\documentclass[letterpaper, journal]{IEEEtran}

\usepackage{amsmath,amssymb,amsfonts}
\usepackage[ruled,vlined,linesnumbered]{algorithm2e}
\usepackage[utf8]{inputenc}
\usepackage{bm}
\usepackage{textcomp}
\usepackage{stfloats}
\usepackage{url}
\usepackage{verbatim}
\usepackage{graphicx}
\usepackage{siunitx}
\usepackage{balance}
\usepackage{subcaption}
\usepackage{tabularx}
\usepackage{booktabs}  
\usepackage{soul}

\usepackage{amsthm}
\newtheorem{theorem}{Theorem}
\newtheorem{lemma}{Lemma}
\newtheorem{corollary}{Corollary}
\newtheorem{remark}{Remark}

\newtheorem{definition}{Definition}[section]
\newtheorem{assumption}{Assumption}

\DeclareMathOperator*{\diag}{diag}

\usepackage{lipsum}
\usepackage{mathtools}

\newcommand\norm[1]{\left\lVert#1\right\rVert}
\usepackage{fix-cm}  
\everydisplay{\fontsize{9}{11}\selectfont}
\DeclareMathOperator*{\argmax}{arg\,max}

\newcommand{\black}[1]{{\textcolor{black}{#1}}}

\usepackage[pdftex, plainpages=false,hypertexnames=true,pdfnewwindow=true,colorlinks=true,citecolor=blue,linkcolor=red,urlcolor=blue,filecolor=blue]{hyperref}%

\usepackage{titlesec}
\titlespacing*{\section}
  {0pt}                     
  {.1ex plus .3ex minus .2ex} 
  {.1ex plus .2ex}           

\titlespacing*{\subsection}
  {0pt}                     
  {.1ex plus .3ex minus .2ex} 
  {.1ex plus .2ex}           

\titlespacing*{\paragraph}
  {0pt}                     
  {.1ex plus .3ex minus .2ex} 
  {.1ex plus .2ex} 

\DeclareMathSizes{10}{9}{7}{5}
\begin{document}

\title{Sampling-Based Coordination-Informed Multi-Objective Multi-Robot Reinforcement Learning}

\author{ Antonio Marino, Esteban Restrepo, Soon-Jo Chung, Paolo Robuffo Giordano, Claudio Pacchierotti \vspace{-2em}
\thanks{A. Marino is with University of Cambridge --- Cambridge, United Kingdom. E-mail: antonio.marino@cl.cam.ac.uk. E. Restrepo, C. Pacchierotti and P. Robuffo Giordano are with CNRS, Univ Rennes, Inria, IRISA --- Rennes, France. E-mail: \{esteban.restrepo, claudio.pacchierotti, prg\}@irisa.fr. Soon-Jo Chung is with the Division of Engineering and Applied Science, California Institute of Technology --- Pasadena, CA 91125 USA. E-mail: sjchung@caltech.edu.\\ This work was supported by the Chair HARMONY funded by ANR and France 2030 under the reference number ANR-23-IACL-0009.}
}
\maketitle
\begin{abstract}
Multi-robot systems must simultaneously optimize competing objectives while maintaining coordinated behavior. Existing multi-agent reinforcement learning approaches often rely on fixed or centralized coordination, which limits adaptability and violates distributed constraints. This work introduces the Coordination-Informed Multi-Objective Reinforcement Learning (CIMORL) framework, integrating a distributed weight prediction mechanism, a privileged expert training strategy, and theoretical guarantees for Pareto-optimal solutions. We present the base CIMORL method alongside two sampling-based variants, CIMORL-TS (Tree Search) and CIMORL-MPPI (MPPI), which leverage privileged global information during training to enable fully decentralized deployment. Experimental validation in cooperative and adversarial scenarios demonstrates a $21.2\%$ hypervolume improvement and superior policy stability compared to state-of-the-art baselines. Real-world experiments with Crazyflie drones further validate the framework's robustness in resource allocation and multi-attacker multi-defend scenarios under partial observability. The project page is available at https://seaingant.github.io/cimorl-website/
\end{abstract}
\section{Introduction}
Multi-robot systems operating in complex environments must
simultaneously optimize multiple competing objectives while
maintaining coordinated team behavior.
Contemporary robotic applications, from heterogeneous drone
swarms in search and rescue operations to collaborative
manipulation tasks, require balancing conflicting objectives
including task completion time, energy consumption,
communication bandwidth, and collision avoidance.
Multi-objective multi-agent systems (MOMAS) provide
realistic models that capture the complexity of inter-agent
interactions and the multi-dimensional nature of their
objectives~\cite{radulescu2024world}.
While multi-agent reinforcement learning (MARL) has emerged
as a powerful paradigm for robotic team
coordination~\cite{gronauer2022multi}, its extension to
multi-objective settings introduces significant theoretical
and practical complexities.
The fundamental challenge lies in determining appropriate
objective weightings that guide robot teams toward globally
optimal solutions while respecting distributed
decision-making processes, limited communication bandwidth,
and heterogeneous sensor capabilities. Existing multi-objective \black{ multi-agent approaches fail to tackle these challenges properly as they} predominantly rely on fixed trade-off schemes or centralized coordination mechanisms~\cite{felten2024momaland}. \black{These methods} lack adaptability to dynamic environmental conditions, violate distributed execution requirements and suffer from poor scalability.
\black{To address these limitations, we introduce the Coordination-Informed Multi-Objective Reinforcement Learning (CIMORL) framework. Our approach bridges the gap between theoretical optimality and practical deployment constraints by integrating three key innovations: (i) a fully distributed weight prediction mechanism that dynamically adapts objective preferences relying solely on local observations and neighbor communication; (ii) a privileged expert training strategy employing multi-objective extensions of Monte Carlo Tree Search (MCTS)~\cite{browne2012survey} and Model Predictive Path Integral control (MPPI)~\cite{williams2018information}; and (iii) theoretical guarantees for Pareto-optimal solutions through proper weight synchronization among robot subgroups. Our sampling-based variants, CIMORL-TS and CIMORL-MPPI, leverage privileged global information during training, yet execute in a fully decentralized manner under realistic partial observability and communication constraints at deployment.} Ultimately, our framework provides a principled, scalable solution for deploying adaptive, multi-objective robotic teams in complex environments.

\section{Related Works}
\label{sec:related-work}
Multi-objective reinforcement learning (MORL) extends classical RL to optimize multiple, often conflicting, objectives concurrently~\cite{roijers2013survey,vamplew2011empirical}. Traditional reinforcement learning seeks to maximize a single scalar reward, an approach that suffices when task priorities reduce cleanly to one metric. However, collapsing multiple criteria into a fixed linear scalarization risks to mask key trade-offs and producing biased or brittle behavior~\cite{hayes2021practical, van2014multi}. MORL addresses this limitation by representing the reward as a vector, preserving feedback for each objective and shifting the learning target to a set of mutually non-dominated, Pareto-optimal policies. Existing single-agent MORL methods generally follow two paradigms. Outer-loop methods treat each preference weight as a separate task, optimizing them sequentially. While algorithms like MORL/D~\cite{felten2024multi}, PGMORL~\cite{xu2020prediction}, and GPI-LS~\cite{alegre2023sample} improve sample efficiency by sharing information across nearby weights or prioritizing promising preferences, they often require extensive retraining for new targets. Inner-loop methods instead condition a single policy on the weight, amortizing learning across weightings. Approaches such as CAPQL~\cite{lu2023multi} and MOPDERL~\cite{tran2023two} achieve markedly higher sample efficiency for continuous robotic control tasks. Our solution falls in this latter category, finding a continuous weight-dependent policy manifold.

Despite this progress, MORL research has largely focused on isolated agents, overlooking the non-stationarity, partial observability, and credit assignment issues that arise when multiple robots  interact~\cite{yliniemi2016multi}. The intersection of these fields defines multi-objective multi-agent reinforcement learning (MOMARL), where agents interact in a shared environment under vector-valued rewards and must reconcile individual preferences with joint dynamics. This challenge is compounded by the need for coordination under partial observability and non-stationarity, together with optimization across multiple conflicting objectives~\cite{radulescu2024world}. Early approaches to MOMARL often sidestepped its full complexity by scalarizing vector rewards with fixed weights and applying existing MARL algorithms~\cite{mannion2016multi}. While tractable, this produced only a single policy per preference and relied on fully decentralized learning, leading to instability in the face of non-stationarity.

Recent works have begun to explore multi-policy learning in multi-agent settings. For discrete-action domains, MO-MIX~\cite{hu2023mo} utilizes a centralized training with decentralized execution (CTDE) architecture and mixing logic to tackle credit assignment in cooperative tasks. For continuous-action frameworks, algorithms such as MO-AIM~\cite{dixit2023learning} and MOAVOA-MADDPG~\cite{abid2024novel} extend learning to asynchronous agents; however, their outer-loop nature requires training separate policies for fixed preferences, lacking a unified preference-conditioned solution capable of covering the Pareto front in a single training run. Furthermore, while recent continuous-domain approaches like MOMA-AC~\cite{callaghan2025moma} use actor-critic training for centralized coordination, they fail to synchronize preference weights, a mechanism we demonstrate in Section~\ref{sec:prob-statement} is essential for solving global team multi-objective problems. Similarly, the contemporary MOMAPPO framework~\cite{felten2024momaland} adapts multi-agent PPO via weighted-sum scalarization and weight sampling strategies (e.g., OLS~\cite{roijers2017multi}, GPI-LS~\cite{alegre2023sample}). However, these static sampling methods fail to account for dynamic environmental conditions that fundamentally shift optimal preference trade-offs.

Addressing these limitations requires algorithms that provide support for continuous state and action spaces, decentralized execution of preference-conditioned policies, and centralized coordination mechanisms that adaptively reshape weights in response to evolving tasks. Consequently, we introduce a dynamic weight prediction methodology based on agent observations and inter-agent communication~\cite{blondin2020algorithm}. Inspired by graph clustering and opinion dynamics~\cite{aminzare2020cluster, bizyaeva2022nonlinear, zhang2025generalized}, we design a Graph Neural ODE framework~\cite{marino2024lgtc} that distinctly models objective preferences as evolving consensus states for distributed weight prediction. Furthermore, we establish a novel theoretical guarantee, proving that cluster-based synchronization is sufficient and necessary condition for team-wide Pareto optimality when agents specialize in specific sub-objectives.

Complementarily, we propose a training paradigm enhanced by sampling-based planning, which provides high-quality supervisory signals to accelerate convergence toward comprehensive Pareto approximations. We leverage privileged information during training via established search methodologies like MCTS~\cite{wang2012multi, hayes2023monte} and MPPI~\cite{williams2018information, ariizumi2023multi}. To overcome the established limitations of linear scalarization, our final architectural contribution integrates the predicted weights into a Tchebycheff problem formulation. This strategic formulation effectively captures non-convex Pareto fronts while fully preserving the decentralized execution required for scalable multi-robot autonomy.

Collectively, these components form a unified framework that adapts weights via distributed opinion dynamics, exploits search-based supervision, and preserves the decentralized execution required for scalable multi-robot autonomy.
\section{Problem Statement}
\label{sec:prob-statement}
\begin{table}[t]
	\caption{Summary of Key Notation}
	\label{tab:notation}
	\centering
	\begin{tabularx}{0.48\textwidth}{@{} l | X @{}}
		\toprule
		\textbf{Symbol} & \textbf{Description} \\
		\midrule
		$\mathcal{V}_s$, $N_s$ & Team of cooperative agents and its total size \\
		$s \in \mathcal{S}$, $u \in \mathcal{U}$ & Joint environment state and joint control action \\
		$z \in \mathcal{Z}$ & Local partial observation vector \\
		$\mathrm{rw}_i$ & Multi-objective reward vector for agent $i$ \\
		$\mathcal{J}_i^\pi$, $J^\pi$ & Expected return for agent $i$ and global team objective under policy $\pi$ \\
		$V^\pi$, $V_\phi$ & Value function and its neural network approximator \\
		$w$, $\bar{w}$ & Distributed predicted weights and static reference weights \\
		$|\cdot|$, $\|\cdot\|$ & Absolute value (element-wise) and Euclidean norm \\
		$(\cdot)_{|}$ & Column-wise matrix vectorization operator \\
        $\otimes$                 & Kronacker product                                                  \\		$\text{So}_x$             & Softmax attention coefficient on the variable $x$ \\
		\bottomrule
	\end{tabularx}
\end{table}

In this section, we formulate the multi-objective problem
within the Markov games framework, following established
reinforcement learning (RL) literature.
We consider a convex environment $Q \subset \mathbb{R}^{n}$
where a fully cooperative multi-agent team $\mathcal{V}_s$
composed of $N_s$ agents operates.
The system state $s \in \mathcal{S}$ represents the agents'
joint states and evolves according to the dynamics $\dot{s}
	= \Omega(s, u): \mathcal{S} \times \mathcal{U}^{N_s} \rightarrow
	\mathcal{S}$, where $u \in \mathcal{U}^{N_s}$ denotes the
joint action of all agents.
Each agent has access to partial observations $z \in
	\mathcal{Z}$, defined by its sensing set $S_i = \{ x \in Q \mid d(x,v_i) \leq \varrho \}$ with $d$ the euclidean distance, i.e., the
environment confined within a sphere of radius $\varrho$ centered
at that agent. 
We assume that cooperative agents within the sensing set
are allowed to communicate therefore forming a
communication graph \black{$\mathcal{G}$}.

The team is tasked with a multi-objective problem,
where each agent receives an individual reward vector
$\mathrm{rw}_{i}(s, u): \mathcal{S} \times \mathcal{U}
	\rightarrow \mathbb{R}^{K}$, whose $K$ entries quantify
performance across different objectives.
The problem can be formally modeled as a Decentralized
Partially Observable Markov Decision Process (Dec-POMDP)
\cite{oliehoek2016concise}, represented by the tuple
$\langle \mathcal{S}, \mathcal{U},\Omega, \{ \mathrm{rw}_{0},
	\dots, \mathrm{rw}_{N_s} \}, \mathcal{Z}, \gamma \rangle$,
where $\gamma \in [0,1)$ is the discount factor.

Given a policy distribution $\pi_i(z_i)$ for each agent,
the expected discounted return over a finite horizon $T$ is
defined as
\begin{equation}
	\label{eq:objective}
	\mathcal{J}_i = \mathbb{E}_{s_0 \sim P_0,\, s' \sim P,\, u \sim \pi} \bigg[ \sum_{t=0}^{T} \gamma^t \mathrm{rw}_{i}(s_t, u_t) \bigg],
	\quad
	J = \sum_{i=1}^{N_s} \mathcal{J}_i.
\end{equation}
The objective is to find the optimal agent policies
$\pi_i(z_i)$ that maximize the collective team performance
$J$:
\begin{equation}
	\max_{\pi} \quad <{ J_1^\pi, \dots, J_{K}^\pi }>.
\end{equation}
where each $J^{\pi}_i$ is an element of the objective $J^{\pi}$ in~\eqref{eq:objective}.

\black{This type of problem is typically solved by scalarizing the multiple objectives and consequently casting it as a single-objective problem. To this end, we employ the augmented weighted Tchebycheff scalarization~\cite{branke2008multiobjective}. Unlike simple linear scalarization, which fails to find solutions in non-convex regions of the Pareto front, the Tchebycheff approach formulates a min-max problem that can discover the entire Pareto frontier by evaluating solutions based on their worst-performing weighted objective. Therefore, we can reformulate the problem as follows:
\begin{equation}
	\label{eq:Tcheb-problem}
	\max_{\pi} \quad \min_{i} \bar{w}_i J_i^\pi + \rho \sum_{i=1}^{K} J_i^\pi = \max_{\pi} U(\bar{w},J^{\pi})
\end{equation}
for some weights $\bar{w} \in \Delta $ satisfying $\sum_{i=1}^{K}
	\bar{w}_i = 1$ and a small $\rho > 0$.
We call $U(\cdot)$ the team utility. 
The augmentation parameter $\rho$ plays a critical role in this formulation. Standard Tchebycheff scalarization (i.e., when $\rho = 0$) can sometimes yield weakly Pareto optimal solutions due to flat regions in its objective contours. Adding the linear sum term, weighted by a strictly positive $\rho$, creates a monotonic slope along these contours. This prevents the selection of weakly dominated points and mathematically guarantees strict Pareto optimality for the team.}
To characterize the solutions to
problem~\eqref{eq:Tcheb-problem}, we first provide the
following definitions:
\begin{definition}
	(Dominated solution): Consider a multi-objective problem and two joint policy solutions $\pi, \pi'$.
	A policy $\pi$ is said to dominate a policy $\pi'$ if $J_k^\pi \geq J_k^{\pi'}$ for all $k \in \{1, \dots, K\}$, and $J_j^\pi > J_j^{\pi'}$ holds for at least one objective with $j \in \{1, \dots, K\}$.
\end{definition}
\begin{definition}
	(Pareto-front): Given a multi-objective problem, the set of Pareto-optimal solutions---called Pareto-front---is the subset of all solutions that are not dominated by another solution.
\end{definition}

All optimal solutions to the Tchebycheff problem are
Pareto-optimal and also Pareto-complete, meaning they fully
capture the Pareto front even when it is
non-convex~\cite{branke2008multiobjective}.
In contrast, the weighted sum formulation captures only
convex portions of the Pareto front; therefore, there might
be Pareto-optimal solutions that are not solutions of the
weighted sum problem.
However, Tchebycheff problem is NP-hard, therefore intractable. 

On the other hand, the team weights \black{$\bar{w}$} are also crucial
for obtaining high-quality solutions, as they define the
multi-objective trade-off, i.e., how the team evaluates the
relative importance of each objective for the specific
problem.
As we are assuming a dynamic environment, these weights
should adapt to the environment with the agent’s objectives
developing in the environment.
Formally, we can represent the team weights as a function
of the joint observations of all agents, \black{$\bar{w}(z_1, \dots, z_N)$}.
However, under partial observability, joint weights cannot
be directly computed and must instead be estimated locally
by each agent using its own observation and information
exchanged through team communication.

In a cooperative multi-agent setting, this motivates
redefining the problem with \emph{distributed weights}:
\begin{equation}
	\begin{aligned}
		\max_{\pi} \quad & \min_{i} \sum_{j=1}^{N_s} w_{ji} \mathcal{J}_{ji}^{\pi_i}
		+ \rho \sum_{i=1}^{K} \sum_{j=1}^{N_s} \mathcal{J}_{ji}^{\pi_i}
		= \max_{\pi} U(w,\mathcal{J}^{\pi})  ,
	\end{aligned}
	\label{eq:real-objective}
\end{equation}
where $\mathcal{J}^{\pi}_j$ is defined
in~\eqref{eq:objective}, the weights satisfy
$\sum_{i=1}^{K} w_{ji} = 1$ and $w$ represents the stack of
agents' weights. The team maximizes the utility of the worst-performing weighted objective. This incentivizes a balanced performance across all goals, promoting altruistic behaviors where agents support a lagging objective for the overall benefit of the team.
The key challenge that arises is: \textit{How should these weights
be derived in a distributed manner?}

To achieve Pareto-optimal solutions, the team weights must
satisfy specific constraints in a cooperative setting.
Consider a case where each agent can achieve a combination
of the objectives, and each agent $j$ is assigned a weight
vector $w_j$.
Let $\mathcal{J}^*_j$ denote the solution of agent $j$ on
its local Pareto front 
for the given Tchebycheff problem.

For simplicity, consider two agents with respective optimal
solutions $\mathcal{J}^*_{1}, \mathcal{J}^*_{2}$.
In
general, there does not exist a common weight vector $w$
such that
\begin{equation*}
	u\bigg(\begin{bmatrix}
			w_{1i} \\ w_{2i}
		\end{bmatrix}, \begin{bmatrix}
			\mathcal{J}^*_{1i} \\\mathcal{J}^*_{2i}
		\end{bmatrix}\bigg) = U(w, \mathcal{J}^*_{1i} + \mathcal{J}^*_{2i} )
\end{equation*}
Indeed, the left-hand side does not correspond to the
original problem~\eqref{eq:Tcheb-problem} and may over- or
under-estimate points on the Pareto front, as illustrated
in Figure~\ref{fig:pareto-front}.
Therefore, Pareto-optimal solutions for the team can only
be achieved when the weights are shared across agents. \black{For analysis purposes, let's assume the theoretical lower bound of $\mathcal{J}_{ij}$ is $\mathcal{J}_{ij} = 0$. We can formally derive the following Lemma:}
\begin{lemma}
	\label{lemma:sync_w}
   \black{ When the agents can achieve any combination of the $K$
	objectives, i.e., they can attain any objective vector with strictly non-zero elements, Pareto-optimal solutions for the team
	$\mathcal{V}_s$ can be obtained through
	problem~\eqref{eq:real-objective} if and only if the agents' weight
	vectors $w$ are synchronized across the team.}
\end{lemma}
\begin{proof}
	\black{Let $\mathcal{J}^*_{ji}$ denote the optimal return for objective $i \in \{1, \dots, K\}$ achieved by agent $j \in \{1, \dots, N_s\}$. The centralized team problem under the augmented Tchebycheff scalarization minimizes the worst-performing weighted objective over the sum of team returns. To achieve the exact same Pareto-optimal solutions for any arbitrary environment, the decentralized aggregation must perfectly recover this centralized formulation:
	\begin{equation}
        \label{eq:proof_eq}
		\min_i \sum_{j=1}^{N_s} w_{ji} \mathcal{J}^*_{ji} = \min_i w_i \sum_{j=1}^{N_s} \mathcal{J}^*_{ji}.
	\end{equation}
    If the agent weights are synchronized such that $w_{ji} = w_i$ for all $j$, the weight $w_i$ becomes a constant with respect to the summation over the agents. It can be factored out, yielding $\min_i ( w_i \sum_{j=1}^{N_s} \mathcal{J}^*_{ji} )$. This exactly matches the centralized team objective, proving the synchronization to be a sufficient condition.}
    
	\black{Moreover, we can show that if the equality in~\eqref{eq:proof_eq} holds over the entire space of achievable returns, then $w_{ji} = w_i$ must be true. Because we assume the agents can achieve any combination of the $K$ objectives, the variables $\mathcal{J}^*_{ji}$ are independent and can take arbitrary non-zero values. For the outputs of the two $\min$ operators to be identical across this entire independently spanning continuous domain, their respective inner arguments must be algebraically equivalent for every $i$. Therefore, we can focus on the linear combinations equality:
	\begin{equation*}
		\sum_{j=1}^{N_s} w_{ji} \mathcal{J}^*_{ji} = \sum_{j=1}^{N_s} w_i \mathcal{J}^*_{ji} \quad \forall i.
	\end{equation*}
	Rearranging the terms yields:
	\begin{equation*}
		\sum_{j=1}^{N_s} (w_{ji} - w_i) \mathcal{J}^*_{ji} = 0 \quad \forall i.
	\end{equation*}
	Since this equation must hold for any arbitrary combination of returns $\mathcal{J}^*_{ji}$, the coefficients of these independent variables must be strictly zero. This requires $w_{ji} - w_i = 0$, which implies $w_{ji} = w_i$ for all agents $j$ and objectives $i$. This proves the statement of the lemma.}
\end{proof}
\black{
In scenarios where individual agents can only contribute to a specific subset of the $K$ objectives (e.g., due to orthogonal task assignments or spatial constraints), the strict necessity of global weight synchronization established in Lemma~\ref{lemma:sync_w} is relaxed. Instead, agents naturally form coalitions to satisfy different objective subsets. }
\black{
Formally, a coalition partition $\mathcal{C} = \{C_1, \dots, C_{N_c}\}$ divides the team $N_s$ into non-empty, disjoint subsets ($\bigcup_{q=1}^{N_c} C_q = N_s$ and $C_p \cap C_q = \emptyset$ for $p \neq q$) such that agents within the same coalition $C_q$ tackle the same subset of objectives. Therefore, the global team problem can be treated as a set of independent sub-teams, leading to the following relaxation:}

\begin{lemma}
	\label{lemma:sync_c}
	\black{Assume the team $N_s$ is partitioned into a set of coalitions $\mathcal{C}$, where each coalition strictly influences a disjoint subset of the $K$ objectives. Pareto-optimal solutions for the global team can be achieved if and only if agents synchronize their weight vectors locally within their respective coalitions. Consequently, different coalitions do not need to share the same weights, allowing the team's weight vectors to diverge into distinct clusters.}
\end{lemma}
\begin{proof}
	\black{Let $\mathcal{C} = \{C_1, \dots, C_{N_c}\}$ be the coalition partition. Let $\mathcal{J}^*_{ji}$ be the optimal return for objective $i \in \{1, \dots, K\}$ achieved by agent $j \in \{1, \dots, N_s\}$. Because the subset of objectives influenced by each coalition is disjoint (orthogonal), for any given objective $i$, there is exactly one corresponding coalition, denoted $C_{(i)}$, capable of achieving it. For all agents outside this coalition ($j \notin C_{(i)}$), the return is strictly zero ($\mathcal{J}^*_{ji} = 0$).}
	
	\black{ Therefore, to perfectly recover the centralized Tchebycheff objective for the global team, we can restrict the same argument of Lemma~\ref{lemma:sync_w} to each individual coalition:
	\begin{equation*}
		\sum_{j \in C_{(i)}} w_{ji} \mathcal{J}^*_{ji} = \sum_{j \in C_{(i)}} w_i \mathcal{J}^*_{ji} \quad \forall i.
	\end{equation*}
	Rearranging the terms yields:
	\begin{equation*}
		\sum_{j \in C_{(i)}} (w_{ji} - w_i) \mathcal{J}^*_{ji} = 0 \quad \forall i.
	\end{equation*}
	Because the agents within $C_{(i)}$ can achieve arbitrary combinations of their specific active objectives, the variables $\mathcal{J}^*_{ji}$ are non-zero and independent. This requires $(w_{ji} - w_i) = 0$, strictly enforcing $w_{ji} = w_i$ for all agents $j \in C_{(i)}$. Crucially, for any agent $j' \notin C_{(i)}$, their weight $w_{j'i}$ is multiplied by zero. Therefore, its value has no impact on the algebraic equivalence of the team objective. This proves that weight synchronization is a necessary and sufficient condition only \textit{within} the active coalition $C_{(i)}$ for any given objective.} 
	
	\black{Iterating this reasoning across all $N_c$ coalitions and $K$ objectives demonstrates that agents only need to synchronize weights locally within their coalition, proving the lemma.}
\end{proof}
\black{The lemma~\ref{lemma:sync_c} reveals in a multi-agent team $\mathcal{V}_s = \{1, 2, \ldots, N_s\}$ operating under partial observability constraints, where agents may form coalitions $\mathcal{C} = \{C_1, \ldots, C_{N_c}\}$ to optimize different subsets of $K$ objectives, the weights $w$ have to cluster into $N_c$ equal vectors corresponding to the $N_c$ robot multi-objective coalitions. Therefore, in this paper we seek to}:

\begin{enumerate}
	\item Learn a distributed weight prediction model
	      $\mathcal{W}_{\psi}: \mathcal{Z} \rightarrow \Delta$ that
	      maps agent observations to weight distributions over
	      objectives, ensuring intra-coalition weight
	      synchronization while allowing inter-coalition weight
	      diversity.

	\item Develop a reinforcement learning framework that
	      approximates the NP-hard optimization
	      problem~\eqref{eq:real-objective} through sampling-based
	      policy search, enabling uniform exploration of the
	      Pareto front via the weight distribution~$\mathcal{W}$.

	\item Ensure the learned policies satisfy the coalition
	      constraints derived in Lemma~\ref{lemma:sync_c}, where
	      agents within each coalition $C_j \in \mathcal{C}$
	      maintain synchronized weights $w_{C_j}$ while different
	      coalitions may adopt distinct weight vectors.
\end{enumerate}

\black{The solution we propose (CIMORL) addresses all these challenges and operates under partial observability constraints, enabling distributed execution.}

\begin{figure}
	\centering
	\includegraphics[width=0.99\linewidth]{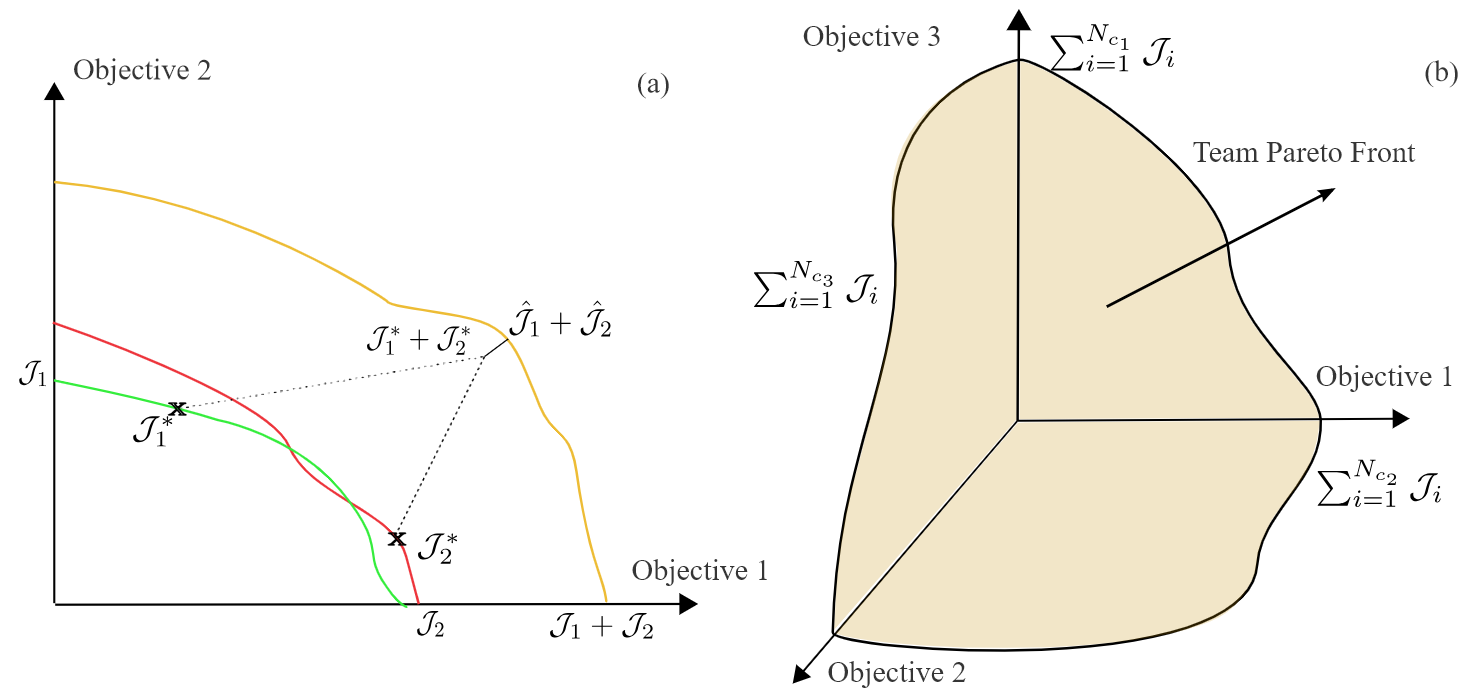}
	\caption{\textbf{Pareto Front representations}: Pareto Front representation for (a) two-objective and two-agent problem and (b) three-objective and a generic team.}
	\label{fig:pareto-front}
\end{figure}

\section{Multi-Objective Weight Prediction}
In this section, we develop a distributed dynamical system
for predicting agent weight distributions that enables
effective multi-objective coordination.
The proposed weight prediction framework addresses two
fundamental requirements: (i) promoting local consensus
among neighboring agents with similar observations to form
cohesive coalitions, and (ii) allowing divergence between
distinct groups to ensure comprehensive coverage of all
objectives across the team.
This design directly satisfies the synchronization
conditions established in Lemma~\ref{lemma:sync_c}, where
agents within coalitions must coordinate their weights to
achieve Pareto-optimal solutions.
When weights naturally cluster into subgroups corresponding
to different objectives, the team achieves improved
objective coverage while maintaining the distributed
structure required by problem~\eqref{eq:real-objective}.
\black{First, we report the proposed opinion dynamics to predict the weights, a core component of our CIMORL approach.}
\subsection{Weight Opinion Dynamics}
For simplicity, we present the model using the stack
observation vector $z$, defined as the concatenation of
individual agent observations in a shared feature space.
However, all computations are implemented in a fully
decentralized manner, with agents relying only on local
observations and neighbor communication. To predict the weights, each agent first processes its local observation through a shared nonlinear mapping $g$:
\begin{equation}
	\label{eq:Xi-def}
	\Xi = g(z) \in \mathbb{R}^{N_s \times K} ,
\end{equation}
where $\Xi$ is used as an observation-dependent preference embedding that simultaneously governs agent clustering, convergence rates, and the final weight distribution.
Rather than directly predicting the final weights, $\Xi$ serves
as a continuous input that influences how the weight
distribution evolves over time based on the current
environmental observations.
The function $g(\cdot)$ learns to map observations to
objective-specific driving signals, enabling the dynamical
system to adapt the weight preferences as the environment
changes.
To ensure numerical stability and bounded dynamics, the
output $\Xi$ is normalized using either min-max
normalization (as employed in our experiments) or a softmax
transformation $\text{So}(\cdot) = \text{Softmax}(\cdot)$.

The weight evolution is modeled as an opinion formation
process where the agents iteratively refine their weight
estimates through neighbor interactions.
Two agents influence each other if their $\Xi$ are
sufficiently similar, formally defined by the condition
$\text{dist}(\Xi_i, \Xi_j) < \epsilon$ for some distance
metric $\text{dist}(\cdot,\cdot)$ and threshold $\epsilon >
	0$.
This similarity-based interaction mechanism naturally
facilitates coalition formation among agents with
compatible objective preferences.

To implement this similarity criterion, we define a
continuous similarity matrix that approaches unity for
similar weight estimates and zero otherwise:
\begin{equation*}
	\sigma_{\Xi} = \sigma\Bigg(W_{\Xi}\bigg(\frac{\Xi \Xi^\top}{||\Xi||^2} - b_{\Xi} \bigg) \Bigg)
\end{equation*}
where $W_{\Xi}$ and $b_{\Xi}$ are trainable parameters that
define the similarity threshold and sensitivity.
The bias term $b_{\Xi}$ establishes a configurable basin of
attraction between agents, allowing adaptation to different
multi-objective task requirements.
The sigmoid function $\sigma$ ensures smooth transitions
between independent signal and coalition attraction
regimes.
Let $S \in \mathbb{R}^{N_s \times N_s}$ represent the adjacency
matrix of the communication graph, encoding which agents
can directly exchange information.
The effective similarity matrix is then constrained by the
communication topology:
\begin{equation*}
	\text{So}_{\Xi} = S \circ \text{So}\bigg(W_{\sigma} \sigma_{\Xi} \bigg)
\end{equation*}
where $\circ$ denotes the Hadamard product and $W_{\sigma}$
provides additional learnable scaling.
This formulation ensures that opinion exchange occurs only
between agents within communication range, maintaining the
distributed nature of the system.

The complete weight dynamics are characterized by a
two-state system comprising a primary weight state $x_2 \in
	\mathbb{R}^{N_s \times K}$ and an auxiliary coordination
state $x_1 \in \mathbb{R}^{N_s \times K}$. \black{This second-order dynamics is strictly required to achieve exact clustering agreements; specifically, $x_1$ acts as an integrator that accumulates disagreement to compensate for the differences between dynamically changing neighboring signals, ensuring $x_2$ converges to a stable consensus.}
The final weight distribution for each agent follows a
Dirichlet distribution with concentration parameters
determined by the softmax transformation of $x_2$.
Inspired by the Liquid Graph Time-Constant (LGTC)
framework~\cite{marino2024lgtc}, we
define the coupled opinion dynamics as:
\begin{equation}
	\begin{cases}
		\begin{aligned}
			\dot{x}_1 = & -(I_N-\text{So}_{\Xi}) x_2 W_B                                \\
			\dot{x}_2 = & -\left(\tau + \Xi \right) \circ x_2-(I_N-\text{So}_{x})x_2W_A \\ & + (I_N-\text{So}_{\Xi})x_1W_B + \Xi B \\
			p_w =       & \text{So}(x_2)
		\end{aligned}
	\end{cases}.
	\label{eq:model-dynamics}
\end{equation}
\black{Here, $W_B$ and $W_A$ are learnable weight matrices that act uniquely as coupling gains between agents to govern the internal state coordination and the network consensus dynamics, respectively. We emphasize that these matrices do not encode any inter-objective trade-offs; their sole purpose is to facilitate network synchronization.}
To ensure positive definiteness and stability, we impose $W
	= \tilde{W}^\top\tilde{W} > 0$ for the system matrices.
The bias parameters $\tau$, $b_{\Xi}$, and $B$ are
implemented as $1_N \otimes b$ to ensure equal biases
across all agents while maintaining learnable adaptation
capabilities.

The system also incorporates an attention mechanism for
$x_2$ communication, defined as:
\begin{equation*}
	\text{So}_{x} = S \circ \text{So}(W_{x}\left(\frac{x_2 x_2^\top}{||x_2||^2}\right))
\end{equation*}
This attention matrix ensures subgroup constant row
sum~\cite{xia2011clustering}, which is essential for proving convergence. This is because the softmax, applied over communicating agents with similar states, effectively selects the agents with similar states leading to potential coalition, which is a key condition for stable clustering~\cite{xia2011clustering}.

In this system, $\Xi$ appears in multiple roles: it
determines the similarity-based coupling structure through
$\text{So}_{\Xi}$, provides damping in the $x_2$ dynamics
through the $(\tau + \Xi) \circ x_2$ term, and drives the
system evolution via the $\Xi B$ input term.
This multivalent role allows the observation-dependent
signal $\Xi$ to simultaneously influence agent clustering,
convergence rates, and the target weight distribution.

Before stating the main theorem of this section, we provide
the following definition of clustered trajectories:

\begin{definition}
	\label{def:cluster-trajectories}
	Given $N$ agent state trajectories $\bar{x}^c = [x^c_1,
		\dots, x^c_N]$, agents are grouped into clusters if they
	follow the same time-evolving state.
	Formally, agents $i$ and $j$ belong to the same cluster
	$C_m$ if and only if $ x^c_i = x^c_j, \quad \forall i, j
		\in C_m, \quad m \in \{1, \dots, M\}.
	$ That is, the cluster set is defined as $ C_m = \left\{ i,j \in \mathcal{V} \;\middle|\; x^c_i = x^c_j \right\}. $ The clustered team state can then be represented as $ \bar{x}^c = [x_{c_1} \otimes \bm{1}_{C_1}, \dots, x_{c_M} \otimes \bm{1}_{C_M} ]^\top $ where $x_{c_m}$ is the representative state trajectory for cluster $C_m$.
\end{definition}

\black{Considering the communication graph $\mathcal{G}$, the following result is derived by making the following assumption on the agent connectivity:} 
\black{\begin{assumption}
\label{ass:assumption-conn}
    We assume the agent communication graph $\mathcal{G}$ is connected throughout the task episode. 
\end{assumption}}

We are now ready to state the theorem about the convergence
properties of the system in \eqref{eq:model-dynamics}.
\begin{theorem}
	\label{th:theorem-2}
	The distributed system in \eqref{eq:model-dynamics}
	clusters in different stable equilibria $\bar{x}^c$ given
	by Definition~\ref{def:cluster-trajectories} and dictated
	by the communication graph topology and observed
	environment.
\end{theorem}

\black{We reported the proof in the Appendix~\ref{sec:app-a}. Figure~\ref{fig:clustering_example} illustrates an example of 20 agents weight concentration parameter trajectories for three objectives randomly initialized (a) and the resulting distributions. The trajectories follow agents' periodic preference embeddings with a different phase per agent. The clustering trajectories are highlighted in different colors and result very stable and distinct. However the trajectories get noisier when the clusters are close as expected since $\text{So}_{\Xi}$ and $\text{So}_{x}$ will have similar value across the agents.}

\black{We note that even if the assumption~\ref{ass:assumption-conn} is not satisfied, the convergence is guaranteed for the connected subgroups.}
\black{While Theorem~\ref{th:theorem-2} guarantees that agents within a connected subgroup will reach a stable consensus on their concentration parameters $p_w$, the dynamical system relies on learnable components, such as the observation mapping $g(\cdot)$, to dictate where these equilibria form. For the team to effectively solve the underlying task, these parameters must be optimized so that the resulting clustered states translate into policies that adequately cover the objective space. Consequently, we must map the converged states to actionable agent weights and establish a training objective that guides the system toward optimal multi-objective performance. This is discussed in the next section.}

\subsection{Weight Distribution
	Training} \label{sec:weight-training}
The concentration parameters $p_w$ define a Dirichlet
distribution $\mathcal{W}_i(w | p_w)$ from which agent
weights are sampled during training and execution. To train the weight distribution $\mathcal{W}$ effectively,
we develop a regret-based loss function that encourages
sampling of high-quality weights while ensuring uniform
exploration of the Pareto front \cite{botros2024regret}.
During training, we estimate the Pareto front
$\mathcal{PF}$ achieved by the current joint policy
$\pi(z,w)$, which depends on the weight samples utilized
during the rollouts.
The set $\mathcal{PF}$ is comprised of the optimal value
functions \black{$V^{\pi, w}$ obtained throughout the training process. In the following, we will use $V^{\pi}$, to indicate a value function given $w$ and the policy $\pi$.} 

We reformulate the Tchebycheff problem as a
minimization to accommodate mixed positive-negative reward
scenarios:
\begin{equation}
	\mathcal{L}_{w}(V) = \sum_{j=1}^{N_s} \max_{i} w_{ji} \frac{|V^{\pi}_{ji} - v^*_i|}{|v^*_i|}
\end{equation}
\black{where $v^*_i$ represents the ideal (utopian) value for objective $i$. Following common practice in empirical multi-objective reinforcement learning~\cite{roijers2013survey}, this reference point is not fixed but continuously estimated online during training. Specifically, $v^*_i$ is defined as the maximum return achieved for each objective across the historical estimated Pareto front $\mathcal{PF}$. To prevent optimization distortion from noisy value estimates early in training, we apply a monotonic update rule after each rollout iteration: $v^*_i \leftarrow \max \left(v^*_i, \max_{V \in \mathcal{PF}} V_{ji} \right)$, ensuring the reference point only shifts when genuinely superior policies are discovered. A small constant $\epsilon = 10^{-6}$ is added to the denominator to prevent division by zero.}
Let $w^*$ denote a set of uniformly sampled weights
available prior to training.
We associate each weight $w^*$ with a value $V^* \in
	\mathcal{PF}$ as the argument that minimizes $\min_{V \in
		\mathcal{PF}} \mathcal{L}_{w}(V)$.
In this manner, we define the optimal weight $w^*$ for the
achievable value $V^{\pi}$.
Let $V^{\pi}(z,w)$ be the value estimated through the rollout
of the current policy for observation $z$.
We project $V^{\pi}(z,w)$ onto the Pareto front to obtain
$V^{* \pi} = \arg\min_{V \in \mathcal{PF}} ||V -
	V^{\pi}(z,w)||$.
The regret incurred by using a sampled weight $\hat{w} \sim \mathcal{W}(\cdot | p_w)$
instead of the optimal weight $w^*$ for the associated optimal
value $V^{*\pi}$ is defined as:
\begin{equation*}
	r(\hat{w} | w^*) = \min_{w^*} \bigg [ \mathcal{L}_{w^*}(V^{*\pi}) - \mathcal{L}_{\hat{w}}(V^{*\pi}) \bigg ].
\end{equation*}
This choice encourages the distribution $\mathcal{W}$ to cover the entire Pareto front, as a sampled weight is only penalized for its distance to its nearest optimal counterpart, not to a single, arbitrarily chosen one.
We therefore minimize the expected regret over the weight
distribution $\mathcal{W}$ to encourage sampling of weights
that yield low regret relative to the optimal Pareto front.
Furthermore, since the weight distribution also influences
the policy, we seek to maximize the expected reward
obtained by the sampled weights.
Given the rollouts dataset $\mathcal{D}$ and the agent advantage estimation
per objective $A^{\pi_i}_{ij}$ \cite{schulman2015high}, we minimize the following loss according to objective \eqref{eq:real-objective}:
\begin{equation*}
	\begin{aligned}
		j_{max} & = \argmax_j \sum_{i=1}^{N_s}w_{ij}\frac{|V_{ij}(z,\hat{w}_i) - v^*|}{v^*}                                                                \\
		\mathcal{L}_{w}   & = \mathbb{E}_{(a,z,\hat{w}) \sim \mathcal{D}} \left[ \log \mathcal{W}(\hat{w} | z) \cdot r(\hat{w} | w^*) \right]              \\
		        & -\mathbb{E}_{(u,z) \sim \mathcal{D}} \Bigg[ \log \mathcal{W}(\hat{w} | p_w) \cdot \sum_{i=1}^{N_s} A^{\pi_i}_{ij_{max}} \Bigg]
	\end{aligned}
\end{equation*}
The preceding loss function employs the log-likelihood
trick for computational convenience. Note that this loss formulation is compatible with the
clipped update mechanism utilized in PPO
\cite{schulman2017proximal}. This loss function balances two goals: minimizing regret to ensure comprehensive exploration of the Pareto front, and maximizing the policy's advantage to exploit high-performing weight regions.

\begin{remark}
	The weight distribution loss can be extended to incorporate
	human preferences for specific objectives.
	This can be achieved by adding a cross-entropy term that
	encourages $\mathcal{W}$ to align with desired objective
	priorities, enabling human-guided multi-objective
	optimization.
\end{remark}

\begin{figure*}[t]
	\centering
	\begin{subfigure}{0.73\linewidth}
		\centering
		\includegraphics[width=\linewidth]{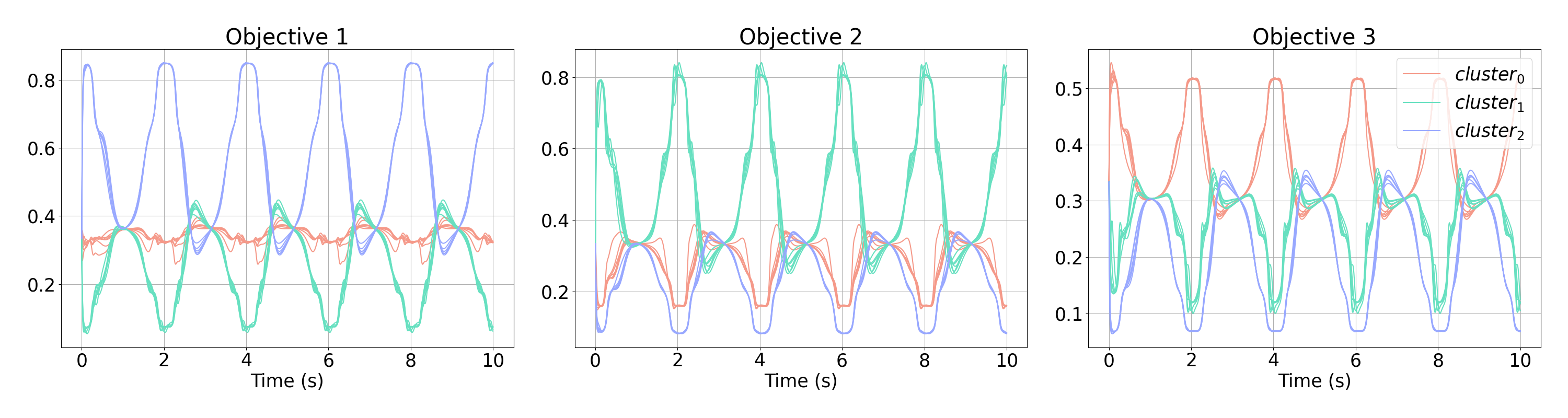}
	\caption{\vspace{1em}}\label{fig:weight_trajectories}
	\end{subfigure}
	\hfill
	\begin{subfigure}{0.25\linewidth}
		\centering
		\includegraphics[width=\linewidth]{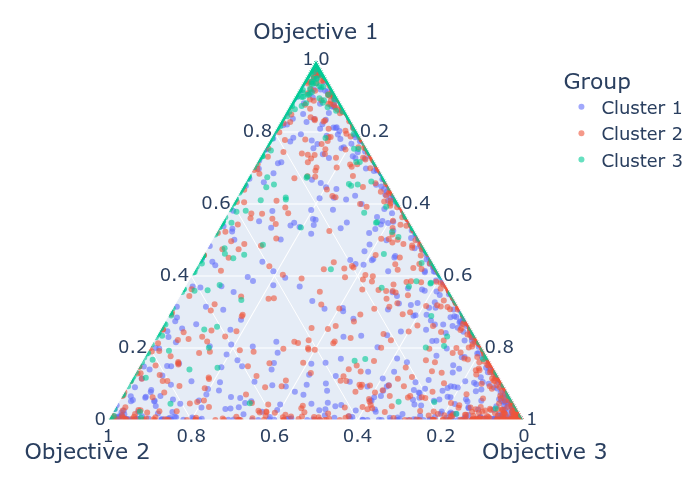} \caption{\vspace{1em}}\label{fig:dirichlet_distributions}
	\end{subfigure}
	\caption{\textbf{Weight clustering trajectories}: (a) \textit{Concentration parameters trajectories.} 20 agents concentration parameters trajectories for $3$ objectives starting from random state and clustering in three subgroups while subjected to a periodic preference embedding with a different phase for each agent. (b) \textit{Final Dirichlet distributions.} Dirichlet distribution for three objectives at the end of clustering trajectories.}
	\label{fig:clustering_example}
\end{figure*}
\section{Policy Training}
\label{method}
In this section, we describe our approach for training
agent policies using local observations and predicted
weights to address the multi-objective optimization problem
in~\eqref{eq:real-objective}.
We introduce two sets of parameters: $\theta$ encoding the
policy approximator $\pi_{\theta}$, and $\phi$ encoding the
value function approximator $V_{\phi}$.
\black{The base CIMORL method trains these networks directly using proximal policy optimization. However,} the training of the weight distribution depends critically
on both the exploration of Pareto-optimal solutions and the quality of the learned policy.
\black{To enhance policy optimality and create our advanced variants, CIMORL-TS and CIMORL-MPPI, we employ a privileged expert framework~\cite{chen2020learning} during training.}
Our approach constructs a high-quality dataset using an
expert sampling methodology that generates approximately
optimal policy distributions for collection in an expert dataset. The privileged expert identifies optimal policies for the entire team by leveraging full system dynamics and environmental state information over a finite planning horizon.

\black{We implement two complementary sample-based expert algorithms to drive these variants: Multi-Objective Neural Tree Search (MONTS)---algorithm~\ref{alg:MCTS}---to generate the CIMORL-TS policy, and Multi-Objective Model Predictive Path Integral (MOMPPI)---algorithm~\ref{alg:MOMPPI}---control to generate the CIMORL-MPPI policy.}

MONTS~\ref{alg:MCTS} extends Neural Tree Expansion (NTE) search~\cite{riviere2021neural} by incorporating predicted weights to scalarize the multi-objective value function. For a given node $n$ in the search tree with accumulated team value $V(n)$, we define the scalarized node value as:
\begin{equation}
	V_{max}(n) \leftarrow \max_i \sum_{j=1}^{N_s}w_{ij} \frac{|V_{ij}(n) - v_i^*|}{v_i^*}
	\label{eq:monte-carlo-value}
\end{equation}
\black{Here, $V_{max}(n)$ represents the worst-case weighted deviation from the ideal performance across all objectives, acting as the scalarized heuristic value for node $n$. The term $V_{ij}(n)$ denotes the accumulated value for agent $j$ on objective $i$ at that specific node, while $w_{ij}$ is the corresponding predicted weight. The parameter $v_i^*$ represents the ideal, utopian reference value (estimated during training), which normalizes the deviations. By minimizing this maximum normalized deviation, the formulation strictly aligns with the Tchebycheff scalarization defined in~\eqref{eq:real-objective}. This design preserves standard tree search exploration rules while enabling branches to evaluate gains across all objectives simultaneously. Algorithm~\ref{alg:MCTS} details this implementation, which crucially returns a new optimal Gaussian action distribution over the root's child nodes rather than selecting a single best action.}

In contrast, MOMPPI extends Model Predictive Path Integral
(MPPI) control~\cite{williams2018information} \black{by integrating the learned neural policy for action sampling during rollout simulations. Given the team's accumulated rewards $V$ across $T$ parallel rollout trajectories, we apply the vectorized Tchebycheff utility function:}
\begin{equation}
	V_{max} \leftarrow \max_i \sum_{j=1}^{N_s}w_{ij} \frac{|V_{ij} - v_i^*|}{v_i^*} \in \mathbb{R}^{T}
	\label{eq:mppi-value}
\end{equation}
Importance weights are computed using a softmax transformation with temperature $\lambda$: $IS \leftarrow \exp(-V_{max}/\lambda) \in \mathbb{R}^{T}$. Unlike traditional MPPI that iteratively updates noisy reference actions, our approach leverages the pre-trained policy $\pi_{\theta}$ to directly yield near-optimal action distributions, eliminating the need for iterative refinement.

During training, both algorithms act as privileged experts, using equal compute and full information for deep searches. At runtime, agents execute computationally-bounded, localized strategies. Using their local observation vector $z_s$, each agent reconstructs a partial team state ($n_o$ for MONTS, $s_o$ for MOMPPI) based solely on visible neighbors. We introduce both methods because their fundamental mechanisms present complementary trade-offs. MONTS prioritizes the deep expansion of the most promising search paths through selective node evaluation. Consequently, MONTS can achieve faster theoretical convergence, with approximation errors decreasing as $O(1/T^{1-c_3/c_2})$~\cite{shah2020non} compared to MOMPPI's $O(1/\sqrt{T})$~\cite{williams2018information}, but incurs higher sequential computation costs. Conversely, MOMPPI performs broad, uniform exploration across the action space, trading sample efficiency for massive parallelization that scales exceptionally well on modern GPU hardware. As illustrated in Figure~\ref{fig:comparison_mppi_mcts}, the optimal choice between these experts depends entirely on the environment's complexity and the target system's parallelization capabilities, making a simple tabular comparison inadequate. Given a weight sample $w \sim \mathcal{W}_{\psi}(z)$ from the agent's predicted distribution, we generate multiple trajectory rollouts by sampling actions from the policy $\pi_{\theta}$. Using these trajectories, we compute the optimal policy distribution through one of our sampling search (SS) methods:
\begin{equation}
	\pi^* \leftarrow \text{SS}(z, \pi_{\theta}, V_{\phi}, w, T, D)
\end{equation}
where $T$ denotes the number of rollout samples and $D$
represents the planning horizon depth.
The complete dataset $\mathcal{D}$ stores the
expert-derived optimal policy distribution $\pi^*$, the
current joint policy $\pi_{\theta}$, rollout observations
$z$, joint actions $u$, value function estimates
$V_{\phi}$, and the corresponding weights $w$.

At the $k$-th training iteration, we assume Gaussian policy
distributions where parameters $\theta$ encode both the
mean $\mu_{\theta}$ and standard deviation
$\sigma_{\theta}$.
The policy training for agent $i$ minimizes the following
supervised learning objective:
\begin{equation*}
	\begin{aligned}
		J_{\pi^*_i}(\theta) =
		\mathbb{E}_{(u,z) \sim \mathcal{D}} \Bigg[ & ||\mu^*_i - \mu_{i_{\theta}}(z)||^2 + ||\sigma^*_i - \sigma_{i_{\theta}}(z)||^2  \Bigg]
		\\ & - \alpha_{\mathcal{H}} \mathcal{H}(\pi_{i\theta})
	\end{aligned}
\end{equation*}
where $\mathcal{H}(\pi_{i\theta})$ represents the policy
entropy and $\alpha_{\mathcal{H}}$ is a temperature
parameter controlling policy stochasticity to maintain
adequate action space exploration.
To accelerate initial convergence, we incorporate an
auxiliary loss term based on the scalarized advantage
function:
\begin{equation*}
	\begin{aligned}
		j_{max}           & = \argmax_j \sum_{i=1}^{N_s}w_{ij}\frac{|V_{ij}(z,w_i,\phi) - v^*|}{v^*}                                                  \\
		J_{\pi_i}(\theta) & = -\mathbb{E}_{(u,z) \sim \mathcal{D}} \Bigg[ \pi_{i}(z_i, \theta) \cdot A_{ij_{max}}^{\pi_i(\cdot, \theta_k)} \Bigg]
	\end{aligned}
\end{equation*}
The final loss for policy training combines both objectives:
\begin{equation}
\label{eq:eqpert-loss}
	J_{\pi_i}(\theta) = (1-\alpha^*)J_{\pi^*_i}(\theta) + \alpha^*J_{\pi_i}(\theta)
\end{equation}
with $\alpha^*$ decaying from 1 to 0 over the training process. \black{This supervised expert loss provides a strong implicit regularization effect. The learned policy is trained to predict the expert coordination strategies at deployment while relying strictly on local, partial observations without requiring the expert's heavy computational overhead.} The value function network $V_i(z,w_i,\phi)$, which outputs the vector of objective values, is trained by minimizing the squared $\mathcal{L}_2$ norm of the temporal difference error across all objectives:
\begin{equation*}
	\begin{aligned}
		 & J_{V_i}(\phi) = \\ &  \mathbb{E}_{ (u,z) \sim \mathcal{D}} \Bigg[ \left\| \left(V_i(z,w_i,\phi) - \left(A_i^{\pi_i(u_i, \theta_k)} + V_i(z,w_i,\phi_k)\right)\right) \right\|^2 \Bigg]
	\end{aligned}
\end{equation*}
\begin{algorithm}[!t]
	\caption{ MONTS}
	\label{alg:MCTS}
	\SetKwFunction{Search}{Search}
	\SetKwFunction{TreePolicy}{TreePolicy}
	\SetKwFunction{Backup}{Backup}
	\SetKwFunction{GetDistribution}{GetDistribution}
	\SetKwFunction{Child}{Child}
	\SetKwFunction{Expand}{Expand}
	\DontPrintSemicolon
	\SetAlgoLined
	\KwData{$z_s, \theta,\phi, w$, $T$ samples, $D$ horizon}
	\KwResult{Find the initializer.}
	create root node $n_0$ with state $z_s$\;
	store the initial value prediction $ v_0\leftarrow V_{\phi}(z_s)$ \;
	$n_i \leftarrow n_0$ \;
	\For{$t = 1 \dots T$}   {
		$\mathrm{rw} \leftarrow $ [ ]\;
		\For{$i = 1 \dots D$}{
			$n_i,\mathrm{rw}_i \leftarrow \TreePolicy(n_i)$\;
			stores $\mathrm{rw}_i$ in the list $\mathrm{rw}$
		}
		append the terminal state value $V_{\phi}(z_s(n_i))$ to $\mathrm{rw}$\;
		\Backup($n_i$, $\mathrm{rw}$)\;
	}
	\Return $\pi^* \leftarrow \GetDistribution(n_0)$ \;

	\textbf{function} \TreePolicy{$n$}\;
	\Indp
	\eIf{$n$ not fully expanded}{
		\Return $\Expand(n)$\;
	}{
		$\displaystyle
			n \leftarrow \arg\max_{n' \in \mathrm{Child}(n)} \Bigl[
				\frac{V_{max}(n')}{N(n')} + c_1 \frac{N(n)^{c_3}}{N(n')^{c_2}}
				\Bigr]$\;
		\Return \Expand($n$)\;
	}
	\Indm

	\textbf{function} \Expand{$n$,$w$}\;
	\Indp
	get observation vector $z_n \leftarrow n$\;
	sample an action $u \sim \pi_{\theta}(z_n,w)$ \;
	create a new child node $n' \leftarrow f(z_n'|z_n,u)$\;
	compute the reward $\mathrm{rw}(z_n')$\;
	\Return $n'$, $\mathrm{rw}(z_n')$\;
	\Indm

	\textbf{function} \Backup{$n, \mathrm{rw}$}\;
	\Indp
	\While{$n \neq \text{null}$}{
		$d \leftarrow \textsc{Depth}(n)$ \;
		\For{$i = d \dots D$}{
			$V(n) \leftarrow V(n) + \gamma^d\mathrm{rw}[d]$\;
		}
		$N(n) \leftarrow N(n) + 1$\;
		$n \leftarrow \textsc{Parent}(n)$\;
	}
	\Indm

	\textbf{function} \GetDistribution{$n$}\;
        \label{alg:MCTS:line:action_distribution}

	\Indp

	$\text{N} \leftarrow \sum_{c \in \textsc{Child}(n)} N(c)$\;
	$\mu \leftarrow \frac{1}{\text{N}} \sum_{c \in \textsc{Child}(n)} N(c) \cdot W(c)$\;
	$\sigma^2 \leftarrow \frac{1}{\text{N}} \sum_{c \in \textsc{Child}(n)} N(c) \cdot (W(c) - \mu)^2$\;
	\textbf{return} $(\mu,  \sqrt{\sigma^2})$ \;
    \Indm
\end{algorithm}

\begin{algorithm}[!t]
	\caption{ MOMPPI}
	\label{alg:MOMPPI}
	\DontPrintSemicolon
	\SetAlgoLined
	\KwData{$z_s, \theta,\phi, w$, $T$ samples, $D$ horizon}
	\KwResult{Optimal policy $\phi^*$}
	Sample $T$ control actions $u_0 \sim \pi_{\theta}(z_s, w)$\;
	Reconstruct the initial state $s_0$ from $z_s$\;
	Initialize cumulative reward $V =0$\;
	\For{$i = 1$ \KwTo $D$}{
		Get new state from environment simulation and reward $s_i, \mathrm{rw}_i \leftarrow f(s_{i-1}, u_{i-1})$\;
		Get new observation $z_{s_i} \leftarrow h(s_i)$\;
		Accumulate reward $V \leftarrow V + \gamma^i \mathrm{rw}_i$\;
		Get action $u_i \sim \pi_{\theta}(z_{s_i}, w)$\;
	}
	Add terminal reward $V \leftarrow V + \gamma^d V_{\phi}(z_{s_D})$\;
	Compute MPPI importance weights $IS \leftarrow \exp(-V_{max} /\lambda) \in \mathbb{R}^{T}$\;
	return $\pi^* \leftarrow \GetDistribution(u_0, IS)$\;
	\textbf{function} \GetDistribution{$u_0$, $IS$}\;
	\Indp
	$W \leftarrow \sum_{i=1}^{T} IS_i$\;
	$\mu \leftarrow \frac{1}{W} \sum_{i=1}^{T} IS_i \cdot u_{0,i}$\;
	$\sigma^2 \leftarrow \frac{1}{W} \sum_{i=1}^{T} IS_i \cdot (u_{0,i} - \mu)^2$\;
	\textbf{return} $(\mu,  \sqrt{\sigma^2})$\;
	\Indm
\end{algorithm}

\begin{figure}[!t]
	\centering
	\includegraphics[width=0.6\linewidth]{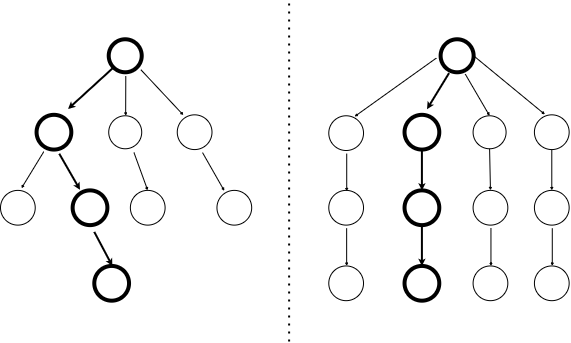}
	\caption{\textbf{Comparison of MPPI and NTS}: on the left, the selection strategy of the Neural Tree Search (NTS); on the right, the Monte-Carlo sampling of Model Predictive Path Integral control.}
	\label{fig:comparison_mppi_mcts}
\end{figure}

\section{Closed-Loop Stability}
In this section, we analyze the closed-loop stability of
the overall system when the multi-agent team employs the
learned policy $\pi_{\theta}(z,w)$.
Since the observations are part of the environment state
and the team state, we can focus only on the overall team
dynamics, i.e. where the agent interactions happen.
We will make two further assumptions to provide closed-loop
stability insights.
\begin{assumption}
	\label{ass:control-struct}
	Given the joint team observation $z$, The joint policy is
	described by $\pi_{\theta}(z,w) = \mu_{\theta}(z,w) +
		\sigma_{\theta}(z, w) \epsilon$ with $\epsilon \sim
		\mathcal{N}(0, I)$.
	Moreover, $\sup_{z,w}|| \sigma_{\theta} || \leq \sigma_g$.
\end{assumption}
\begin{assumption}
	\label{ass:ito-stoc}
We assume that the robots' observation transition dynamics are deterministic, which is a common simplification for tractable stability analysis, and thus attribute all stochasticity to the control policy. The system dynamics under control $\pi$ can therefore be described by the $it\hat{o}$ stochastic differential equation:
\begin{equation}
		\label{eq:ito-stochastic}
		dz = f(z,t)dt + \mu_{\theta}(z,w)dt + \sigma_{\theta}(z,w) \epsilon
	\end{equation}
\end{assumption}
Note that this system model also captures the state
transition experienced in the sampling search.
Since the weights $w$ serve as objective selectors that
focus the observation vector on the objectives while
preserving the underlying processing functions, we can
analyze the closed-loop stability by considering each
weight-induced subsystem individually.
This approach leverages contraction theory~\cite{lohmiller1998contraction} for switching
systems~\cite{russo2012contraction}, where each 'subsystem' corresponds to the agent dynamics \eqref{eq:ito-stochastic} under a fixed weight vector $w$. The overall system's stability is then analyzed as these weights evolve over time.
For the stability analysis, we require a contraction matrix
$M(z,t) \succ 0$ that is Lipschitz continuous with constant
$L_M$ and satisfies $\underline{m}I \preceq M \preceq
	\overline{m}I$ for positive scalars $\underline{m}$ and
$\overline{m}$.

\black{The matrix $M(z,t)$ defines a state-dependent Riemannian metric. Unlike standard Lyapunov theory which measures the energy or distance of a single trajectory to a fixed equilibrium, contraction analysis utilizes $M(z,t)$ to continuously measure the differential distance between any two neighboring system trajectories. If the vector field of the system ensures that this differential distance is strictly shrinking with respect to the metric defined by $M(z,t)$, all trajectories will exponentially converge toward each other, regardless of their initial conditions.}

Under this contraction matrix, the agent dynamics with
policy $\pi_{\theta}$ satisfies the contraction condition:
\begin{equation*}
	\dot{M} + M^{T}\frac{\partial (f + \mu_\theta)}{\partial z} + \frac{\partial (f + \mu_\theta)}{\partial z}^\top M \preceq -\alpha M - \alpha_s I
\end{equation*}
where $\alpha > 0$ represents the contraction rate,
$\alpha_s = L_M(2\sigma_g^2)/(\alpha_g+0.5)$ accounts for
stochastic perturbations, and $\alpha_g > 0$ is a design
parameter.

Under these conditions, results in stochastic
contraction theory \cite{tsukamoto2021contraction} guarantee exponential convergence in
expectation between any two system trajectories $\xi_1$ and
$\xi_2$:
\begin{equation*}
	\mathbb{E} \left[ \| \xi_2(t) - \xi_1(t) \| \right] \leq \frac{ \mathbb{E} \left[ \int_{\xi_{1|}(0)}^{\xi_{2|}(0)} \delta z^\top M \delta z  \right]}{{\underline{m}}}  e^{-2\alpha t} + \frac{C}{2 \alpha} \frac{\overline{m}}{\underline{m}}
\end{equation*}
where $C = (2\sigma_g^2)/(2 \alpha_g^{-1} + 1)$ bounds the
steady-state error due to stochastic noise.
The overall system's contraction rate is given by $a =
	\min_i \alpha_i$ across all weight-induced subsystems.
Additionally, let $L_w$ denote the Lipschitz constant of
system~\eqref{eq:ito-stochastic} with respect to the weight
parameters $w$.

To establish bounded behavior of the weight dynamics, we first present the following result:

\begin{lemma}
	\label{lemma:bounded-x2}
	The weight state $x_{2}$ of system~\eqref{eq:model-dynamics} remains bounded in the invariant set $[-B,B]^{N \times K}$ for all $t \geq 0$ if initialized within this set. Furthermore, the auxiliary state $x_{1}$ evolves within a bounded set $\mathcal{X}_1$.
\end{lemma}
The proof is reported in the Appendix~\ref{sec:proof-lemma3}. Given the Lipschitz constant $L_Z$ of the function $g$ in~\eqref{eq:Xi-def}, the closed loop system is contracting under the following Theorem:

\begin{theorem}
	\label{theorem:closed-loop}
	Under the joint control policy $\pi_{\theta}$ and weight
	dynamics~\eqref{eq:model-dynamics}, the closed-loop
	multi-agent system is contracting if the small-gain condition
	\begin{equation}
		\label{eq:small-gain}
		\lambda_{\max}(F_{22}) \cdot a > L_{G} \cdot L_{w}
	\end{equation}
	is satisfied, where
	\begin{equation}
		L_{G} = 2 L_Z \left( \|W\|_{\infty} + 1 + \frac{1}{16}\|W_B\|_{\infty} \right)
	\end{equation}
	represents the coupling strength between the weight and
	observation dynamics.

\end{theorem}
We provide a proof in the appendix~\ref{sec:app-b}
\begin{remark}
	The Lipschitz constant $L_Z$ can be readily computed for
	neural networks $g$ employing standard activation functions
	such as ReLU or Tanh.
	Furthermore, $L_Z$ has a well-defined analytical form when
	$g$ is spectrally-normalized~\cite{miyato2018spectral},
	providing practical computational advantages for stability
	verification.
\end{remark}


Alternatively, we can apply the same stability analysis
framework to the weight selection dynamics under $w$.
This leads to the following equivalent characterization:
\begin{corollary}
	The stability condition in
	Theorem~\ref{theorem:closed-loop} is equivalent to ensuring
	that each subsystem of~\eqref{eq:ito-stochastic} induced by
	a fixed weight selection $w$ exhibits contractive behavior,
	along with the weight dynamics
	system~\eqref{eq:model-dynamics} having its own contraction
	rate.
	The reverse implication also holds.
\end{corollary}
\begin{proof}
	This equivalence follows directly from
	Theorem~\ref{theorem:closed-loop} by considering each
	subsystem individually with $L_w = 0$, which decouples the
	weight dynamics from the observation dynamics.
	When $L_w = 0$, the small-gain condition reduces to
	verifying contraction of each subsystem independently.
\end{proof}

The conditions stated in Theorem~\ref{theorem:closed-loop}
and the agent dynamics contraction can be enforced during
training by adding a regularization term to the policy and weight-predictor loss functions that penalizes the spectral norms of the network parameters. This helps control their Lipschitz constants and satisfy the small-gain condition.

\begin{remark}\label{remark:on-ass12} Assumptions~\ref{ass:control-struct} is a standard policy exploration assumptions derived from the learning framework while Assumption~\ref{ass:ito-stoc} is a direct controlled consequence of our learning formulation. However, the deterministic observation dynamics might be restrictive even if our contraction analysis inherently guarantees robustness against bounded additive disturbances~\cite{tsukamoto2021contraction}. We reserve to future works the extension of the formal analysis to encompass general stochastic dynamics and observation.\end{remark}

\section{Evaluations}
\label{sec:evaluations}
\begin{figure}[t]
	\centering
	\begin{subfigure}{0.49\linewidth}
		\centering
        \includegraphics[width=\linewidth]{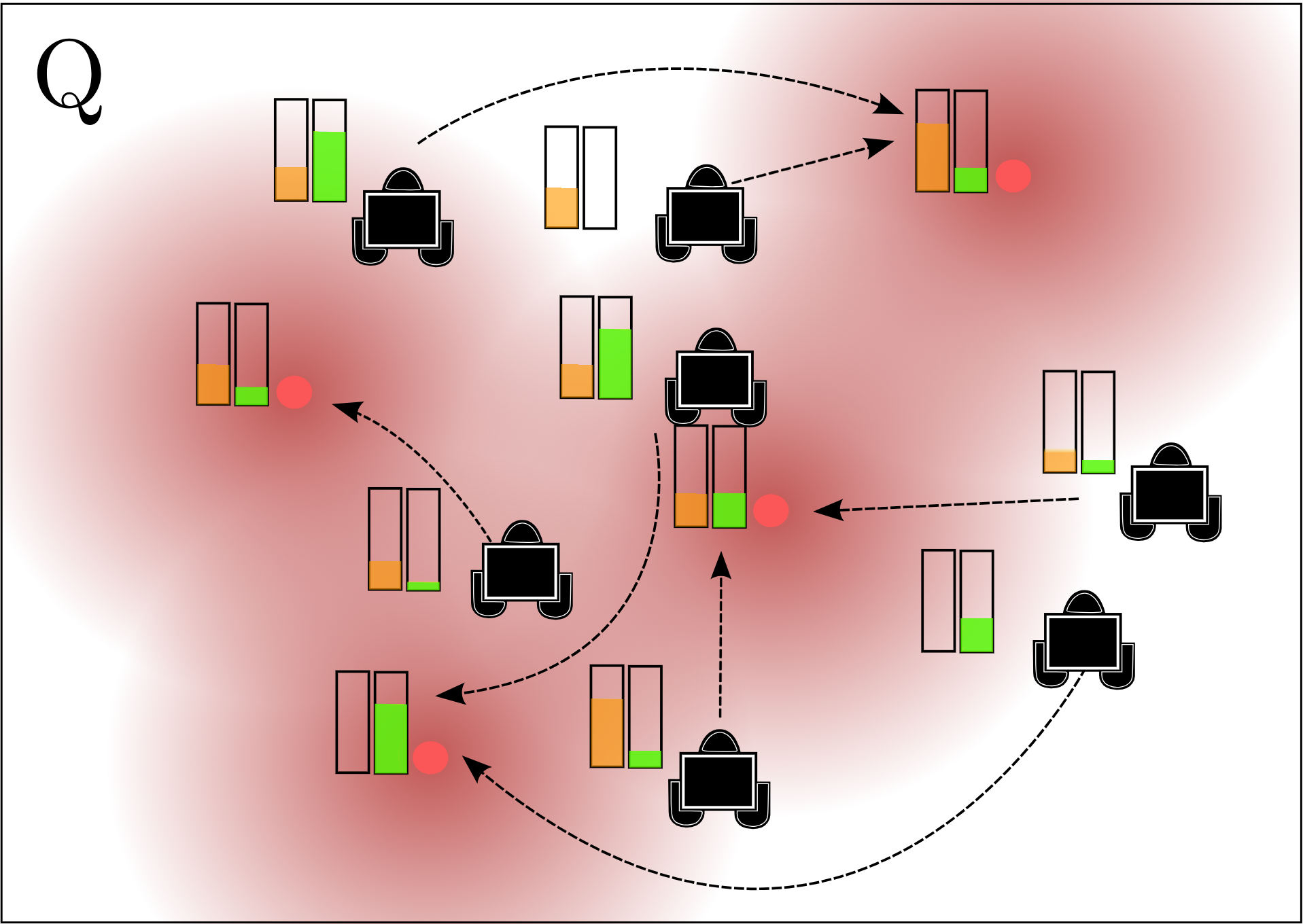} \caption{\vspace{1em}}
        \label{fig:resource_assignment_scenario}
	\end{subfigure}
	\hfill
	\begin{subfigure}{0.49\linewidth}
		\centering
		\includegraphics[width=\linewidth]{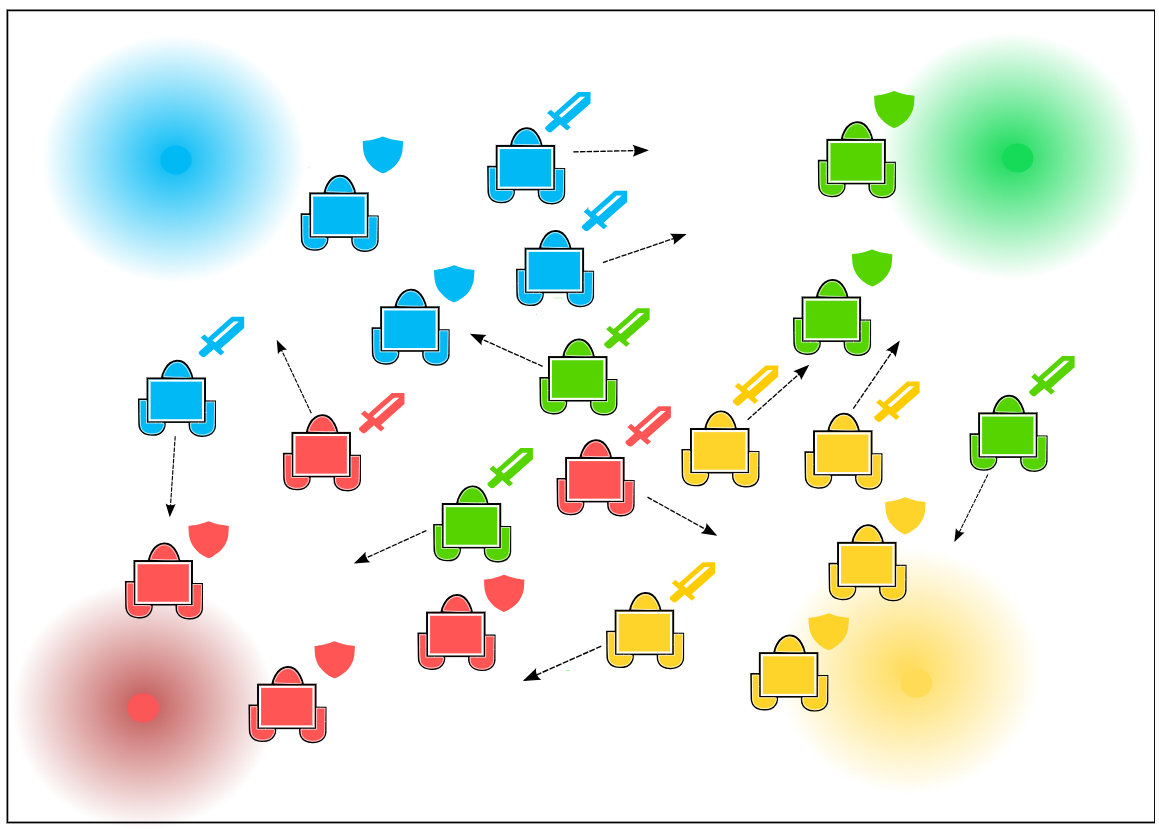} \caption{\vspace{1em}}\label{fig:attacker_defender_scenario}
	\end{subfigure}
	\caption{\textbf{Experimental scenarios}: (a) \textit{Multi-resource assignment scenario.} Agents must deliver resources to consumer locations (red) while avoiding collisions and maintaining connectivity. (b) \textit{Multi-team attacker-defender scenario.} Multiple teams compete against each other to infiltrate multiple opponent safe areas while defending their own territory against opposing teams.}
	\label{fig:scenarios}
\end{figure}
We evaluate our base method, CIMORL, and its sampling-augmented variants (CIMORL-TS and CIMORL-MPPI) on two multi-agent domains: multi-resource assignment and multi-team attacker-defender. Detailed hardware, training, and architectural setups are deferred to the Appendix.

We compare against state-of-the-art multi-policy baselines MOMAPPO~\cite{felten2024momaland} and MO-MIX~\cite{hu2023mo}, which generate discrete sets of policies. To ensure a fair algorithmic comparison against our continuous conditioned approach, we evaluate all methods over a uniformly sampled batch of evaluation preference weights $w^*$. For the multi-policy baselines, for each weight vector in the evaluation batch, we select the specific policy from their generated set that maximizes the expected scalarized utility under that weight. Furthermore, we introduce MOMAPPO-TS, an on-policy baseline augmented with Tree Search that incorporates privileged expert data via the loss defined in Eq.~\eqref{eq:eqpert-loss}. This further baselines serves as further validation of our method when the baseline is trained under the same privileged expert conditions.

To evaluate the quality, coverage, and distribution of the solutions, we report the hypervolume indicator~\cite{guerreiro2021hypervolume}, spacing, and diversity metrics~\cite{hu2023mo}. Policy adaptability and performance are assessed via the Expected Scalarized Return (ESR) formulation of expected utility~\eqref{eq:real-objective}. Because exhaustive empirical rollouts in continuous multi-objective spaces are intractable and suffer from high variance, we efficiently approximate the Pareto front by evaluating the trained value function $V_{\phi}$ over identical, uniformly sampled weights $w^*$ consistently across all methods~\cite{felten2024momaland}. The final front is then refined using a pruning K-means clustering method~\cite{petchrompo2022review}.

Furthermore, to empirically validate Lemma~\ref{lemma:sync_w} and Lemma~\ref{lemma:sync_c}, we evaluate two ablation baselines derived from our method: CIMORL w/o Sync (which bypasses consensus, directly mapping the embedding $\Xi$ to the weight distribution) and CIMORL w/o Clust (which sets $\text{So}_{\Xi} = S$ and $\text{So}_{x} = S$ in Eq.~\eqref{eq:model-dynamics} to force a single global consensus rather than distinct coalitions).

The following sections detail each scenario, presenting both the simulated training results and the real-world experimental validation using a team of Crazyflie 2.1 drones.
\subsection{Multi-Resource Assignment}
\label{sec:resource-assignment}
We consider a multi-resource assignment scenario involving a fully cooperative multi-agent team $\mathcal{V}_s$ comprising $N_s$ agents. The task requires delivering $r$ distinct resource types to $K$ consumer locations situated within a confined operational space $Q$, where each consumer may demand a specific subset of the available resources. Each consumer is associated with an interaction region of radius $R_k$, within which an agent may release its carried resources to satisfy the consumer's demand, as illustrated in Figure~\ref{fig:scenarios}. Furthermore, agents experience collisions when their inter-agent distance falls below a threshold $R_{cl}$.

The information structure enforces partial observability constraints, where each agent $i$ has access to limited observations $z_i \in \mathcal{Z}$ comprising: (i) its current position $p_i$, velocity $v_i$, and carried resources $\mathrm{so}_i \in \mathbb{R}^r$; (ii) the locations of all target areas $ds_k$ and their corresponding demanded resources $\mathrm{de}_k \in \mathbb{R}^r$; and (iii) a restricted set of teammate positions $\mathcal{O}_s = \{p_j \mid j \in \mathcal{N}_i\}$ within the agent's sensing range. To ensure the existence of a unique resource allocation solution, we assume that the aggregate available supply $\sum_{i=1}^{N_s} \mathrm{so}_i$ does not exceed the total demand $\sum_{k=1}^{K} \mathrm{de}_k$. The objective of this scenario is to determine a collaborative control policy that enables the agents to optimally satisfy the resource demands within the environment while maintaining network connectivity and avoiding inter-agent collisions. These objectives inherently conflict; optimizing delivery speed requires spatial dispersion that compromises connectivity, while aggressive movement increases spatial density and collision risks.

Each agent receives a vector reward $\mathrm{rw}_{i}(s, u): \mathcal{S} \times \mathcal{U} \rightarrow \mathbb{R}^{K+2}$ where the first $K$ components reward the successful fulfillment of the $K$ consumer demands, and the final two components incentivize connectivity maintenance and collision avoidance, respectively. Note that the reward $\mathrm{rw}_i$ depends on the joint state and actions of all agents, reflecting the inherently cooperative nature of the task. Specific details regarding the mathematical formulation of these rewards and the neural network architectures are deferred to the Appendix.
\begin{figure*}[t]
	\centering
	\includegraphics[width=\linewidth]{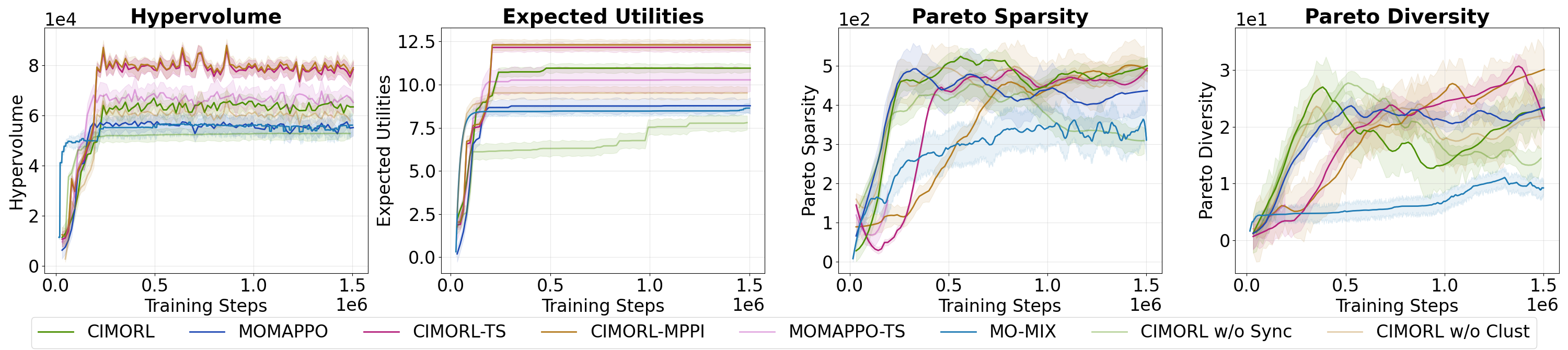}
	\caption{\textbf{Multi-resource assignment training comparison}: Hypervolume indicator and expected utility results for all methods across multiple training runs.}
	\label{fig:comparison_ra}
\end{figure*}
\begin{figure}[t]
    \centering
    \includegraphics[width=0.8\linewidth]{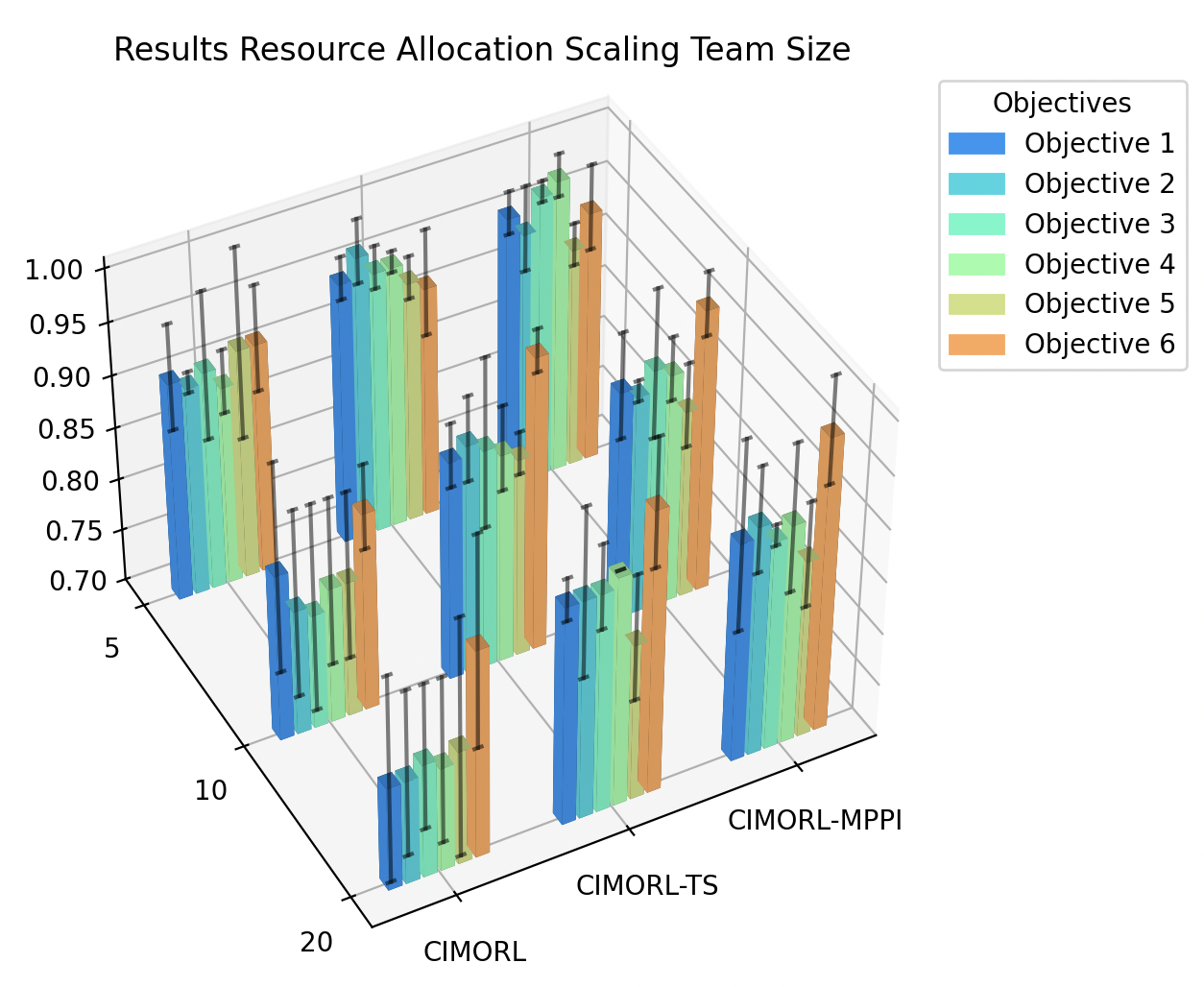}
    \caption{Normalized return for resource-allocation task for number of agents in the team $N=[5,10,20]$ and $6$ objectives ($4$ demanding locations, obstacle avoidance and connectivity) reported as bar plots with their standard deviation as black lines.}
    \label{fig:ra-team-size}
\end{figure}
\begin{figure}[t]
	\centering
	\includegraphics[width=\linewidth]{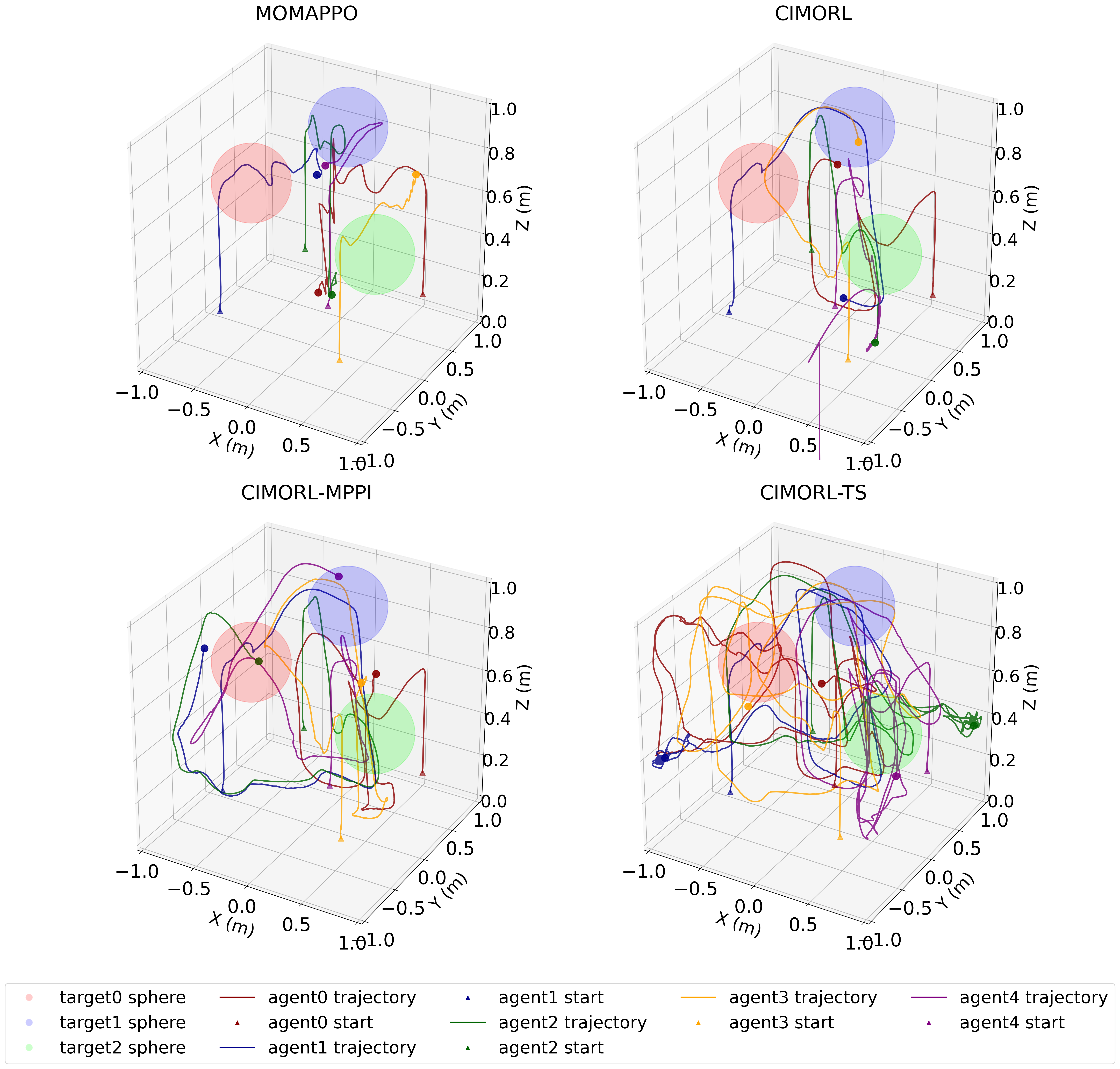}
	\caption{\textbf{3D trajectories in multi-resources allocation experiment}: 3D trajectory comparison of the four methods.}
	\label{fig:ra_3d_trajectories}
\end{figure}
\begin{figure}[t]
	\centering
	\includegraphics[width=\linewidth]{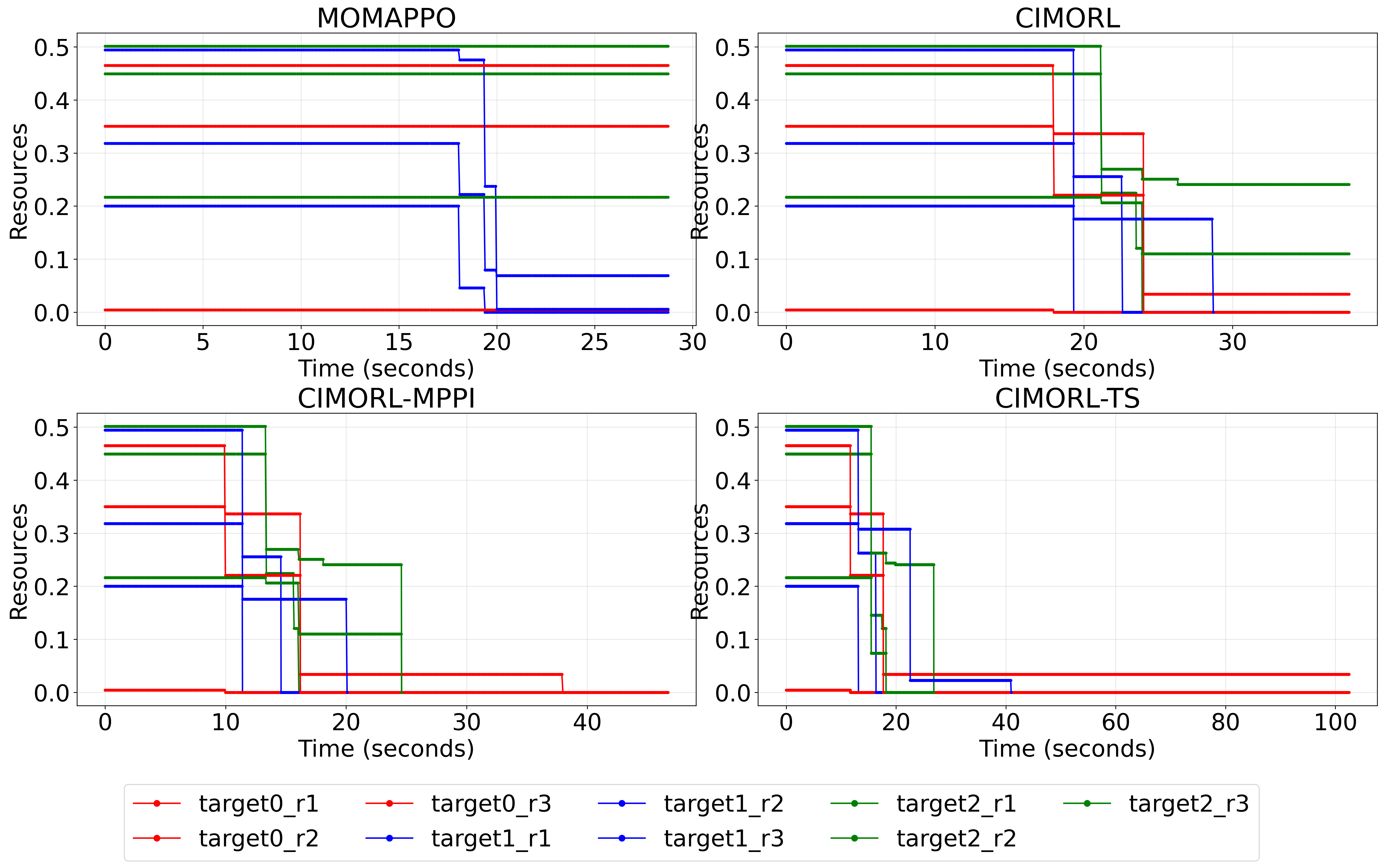}
	\caption{\textbf{Demanding resource evolutions in multi-resources allocation experiment}}
	\label{fig:ra_targets}
\end{figure}
\begin{table}[t]
    \centering
    \caption{Disconnection rate in multi-resource allocation experiment}
    \begin{tabular}{llll}
        \toprule
        \textbf{MOMAPPO} & \textbf{CIMORL} & \textbf{CIMORL-MPPI} & \textbf{CIMORL-TS} \\
        \midrule
        $0.0\%$ & $0.0\%$ & $0.2\%$ & $2.8\%$ \\
        \bottomrule
    \end{tabular}
    \label{tab:disconnection_analysis}
\end{table}
\textbf{Training Results.} For the multi-resource assignment scenario, we conducted training with $5$ agents, $3$ resource types, and $4$ target locations. All sampling-based algorithms were trained with identical simulation budgets and planning horizons to ensure fair comparison. The learning rate was scheduled to decrease during training to promote convergence stability.

Figure~\ref{fig:comparison_ra} presents the training comparison across all methods using four key metrics: \black{Pareto Front hypervolume, sparsity and diversity,} and expected utility. The results demonstrate clear performance distinctions between the proposed methods and baseline approaches.

The sampling-based variants of our approach, \black{CIMORL-MPPI} and \black{CIMORL-TS}, achieved superior performance with statistically equivalent results. Both methods attained hypervolume values of approximately $8.0 \times 10^4 \pm 4.0 \times 10^3$, \black{indicating excellent Pareto front coverage, sparsity and diversity.} Their expected utility performance was similarly strong, reaching $12.0 \pm 0.5$, demonstrating consistent policy quality across multiple evaluation runs. \black{CIMORL-MPPI and CIMORL-TS perform comparably, proving that both MPPI and Monte Carlo Tree Search offer effective expert guidance. With $2000$ simulation samples and a horizon of $25$, there is no statistically significant difference between the two approaches. }

The base CIMORL method, while employing the same weight prediction framework, achieved moderately lower performance with a hypervolume of $6.4 \times 10^4 \pm 6.0 \times 10^3$ and expected utility of $11.0 \pm 0.3$. This performance gap of approximately $20\%$ in hypervolume compared to the sampling-based variants highlights the importance of privileged expert guidance during training. Nevertheless, CIMORL still outperformed the baseline methods, demonstrating the effectiveness of the proposed weight prediction mechanism. All these methods reach similar levels of Pareto Front and Sparsity, showing high diversity of the weight that span very different points on the Pareto Front.

Among the baseline methods, MOMAPPO-TS showed improved performance over standard MOMAPPO, achieving a hypervolume of $6.6 \times 10^4 \pm 8.0 \times 10^3$ compared to MOMAPPO's $5.8 \times 10^4 \pm 5.0 \times 10^3$. This $13.8\%$ improvement demonstrates that tree search integration provides meaningful benefits even for fixed-weight approaches. Moreover, the improvement of $21.2\%$ in hypervolume achieved by \black{CIMORL-MPPI} and \black{CIMORL-TS} over MOMAPPO-TS shows the impact of the weight prediction model.

However, all baseline methods exhibited substantially lower expected utility performance, with MOMAPPO achieving $8.5 \pm 0.5$, similar to MO-MIX and MOMAPPO-TS reaching $10.5 \pm 1.0$. The significantly higher variance in MOMAPPO-TS's expected utility ($\sigma = 1.0$) indicates inconsistent policy performance, likely due to the fixed weight selection strategy's inability to adapt to varying environmental conditions. Without synchronization, our framework's Pareto Front converges to a low, stable hypervolume alongside a marked decline in sparsity and diversity. This plateau demonstrates that lacking explicit coordination yields stable, albeit lower-quality solutions, as reflected by the expected utilities. Conversely, the dynamic synchronization in CIMORL w/o Clust effectively generates adaptive, high-value solutions. This configuration achieves a hypervolume of $6.0 \pm 0.5 \times 10^4$ and expected utilities of approximately $9.0 \pm 0.4$, outperforming MOMAPPO while remaining subordinate to the complete CIMORL model. Furthermore, this variant maintains sparsity and diversity levels comparable to both CIMORL and its sampling-based counterpart. Ultimately, while this method produces an evenly distributed Pareto Front, its solutions remain strictly dominated by those discovered by the full CIMORL framework. 

In Figure~\ref{fig:ra-team-size} , we present the results of scaling the team size from $5$ to $20$ agents. To facilitate comparison, returns were shifted and normalized to $[0, 1]$ based on the maximum return achieved for each objective across all $10$ random seeds. As expected, overall returns decrease with larger teams. However, connectivity (Objective $6$) improves, as the denser environment naturally makes this objective easier to satisfy. In general, CIMORL-TS and CIMORL-MPPI maintain returns above $0.9$ across all objectives, whereas performance drops to approximately 0.8 under the base CIMORL policy. This suggests that the sampling search achieves broader coordination than what is observed during training. Furthermore, performance remains remarkably consistent across the different objectives, demonstrating that roles are evenly assigned and the clustering mechanism effectively distributes tasks among agents as the team size grows.

\textbf{Experimental Comparison.} We performed real lab experiments to evaluate our approach. The experimental setup utilized 50 simulations for the three search-based methods and 1000 simulations for the MPPI-based approach with a horizon length of 10 steps to maintain computation within the identical hardware time budget of $50$~ms. This discrepancy is not an artificial setting but the natural outcome of the methods' inherent mechanisms: MONTS requires sequential node expansion that bottlenecks parallelization, whereas MOMPPI evaluates independent trajectories, allowing massive GPU batching within the same latency constraint.

Figure~\ref{fig:ra_3d_trajectories} reports the 3D trajectories of the experiments while the resource allocation dynamics are illustrated in Figure~\ref{fig:ra_targets}, for the four methods, MOMAPPO, CIMORL, \black{CIMORL-TS}, and \black{CIMORL-MPPI}. The results reveal significant behavioral differences among the methods. MOMAPPO demonstrated the most severe limitations, with agents colliding almost immediately upon mission initiation, preventing any meaningful resource allocation. CIMORL exhibited marginally better performance, allowing agents to operate for a brief period before collision occurred, yet still failing to complete the resource allocation task. 

In contrast, \black{CIMORL-TS} successfully avoided inter-agent collisions throughout the experiment, demonstrating that the sampling-based search mechanism effectively adapts to real drone dynamics. However, despite collision avoidance, \black{CIMORL-TS} failed to achieve complete resource allocation even during extended flight operations. \black{CIMORL-MPPI} emerged as the only method capable of both collision avoidance and successful resource allocation, validating the effectiveness of the proposed coordination-informed approach. \black{ This shows the compromise of CIMORL-TS in real-time applications as we had to drastically reduce the number of simulations to stay within the time budget. Consequently, the asymmetric search tree is computationally starved under finite sampling; it lacks the iterative depth necessary to reliably balance exploration and exploitation, systematically leading to brittle decisions based on unrefined leaf evaluations. Conversely, the 1000 uniform parallel samples achieved by CIMORL-MPPI generate a robust, smoothed approximation of the objective landscape, inherently producing safer, average-case optimal decisions under the identical time constraint.}

Table~\ref{tab:disconnection_analysis} provides insights into the connectivity objective. MOMAPPO and CIMORL both exhibit $0\%$ disconnection rates, as the flights terminate so rapidly due to collisions that insufficient time elapses for agents disconnection to manifest. \black{CIMORL-MPPI} achieves the most favorable disconnection performance with only $0.2\%$ disconnection rate, demonstrating robust coordination capabilities. \black{CIMORL-TS experiences a higher disconnection rate of $2.8\%$. Because agents take longer to fly and spread out further to allocate the remaining resources, the increased distance between them leads to more frequent link failures.}

The experimental results validate several key aspects of our approach: (i) the proposed weight prediction framework consistently outperforms fixed-weight baselines across both metrics; (ii) privileged expert guidance through sampling-based algorithms provides substantial performance improvements; and (iii) both MPPI and MCTS variants achieve comparable results, offering flexibility in computational resource allocation based on specific deployment constraints.

\subsection{Multi-team Attackers-Defenders}
\begin{figure}[t]
    \centering
    \includegraphics[width=0.8\linewidth]{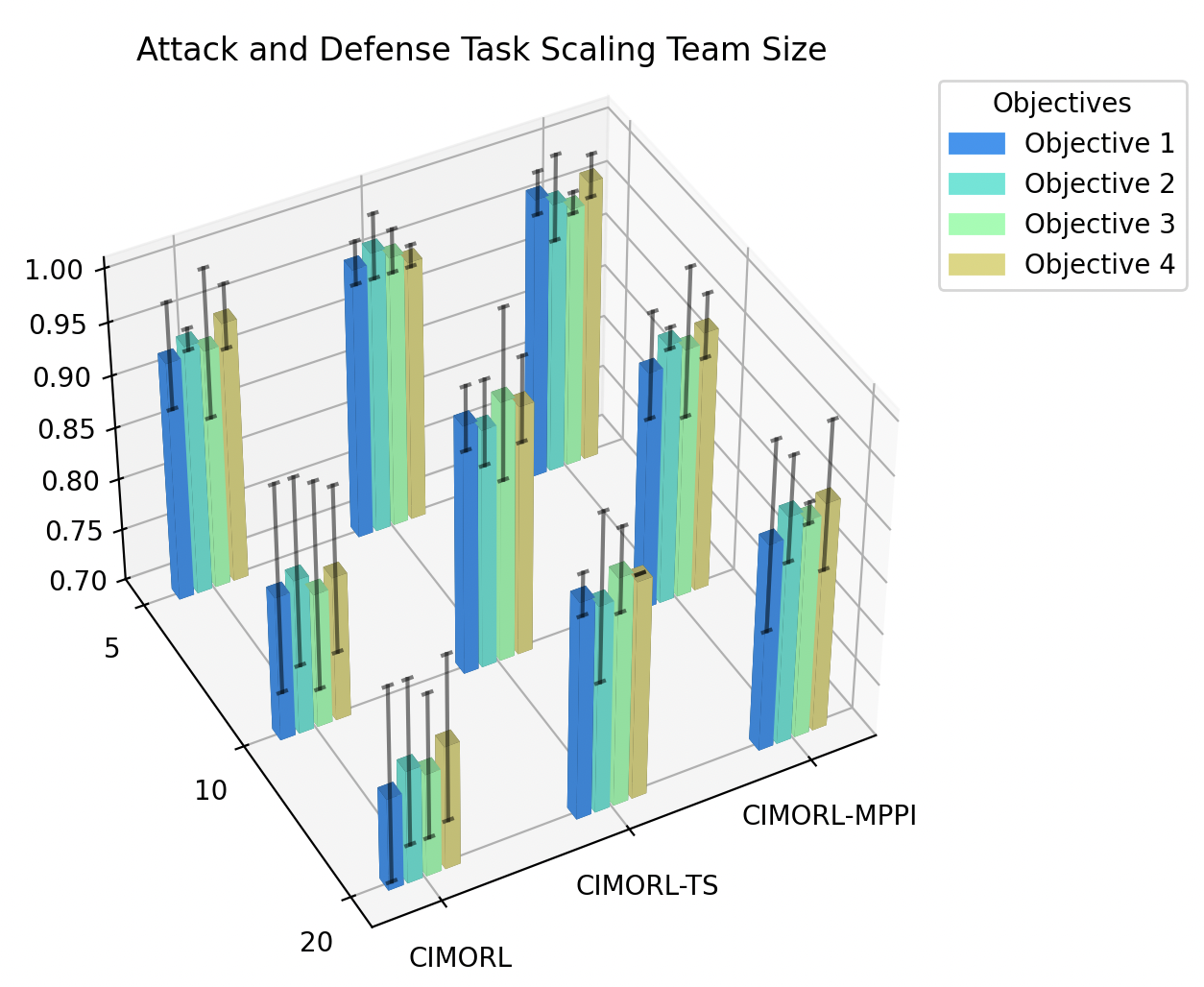}
    \caption{Normalized returns for the attacker-defender task for varying team sizes of $5$, $10$, and $20$ agents, evaluated across $4$ teams ( therefore $4$ objectives). Results are reported as bar plots with their standard deviation as black lines.}
    \label{fig:ad-team-size}
\end{figure}
We consider a complex multi-team adversarial scenario featuring asymmetric attack and defense objectives. The environment consists of a fully cooperative attacking team $\mathcal{V}_s$ comprising $N_s$ agents that must simultaneously accomplish the following objectives: penetrate $K$ distinct opponent safe areas $\mathcal{R}_k$, each defended by teams $\mathcal{V}_k$ of size $N_k$, while protecting their own territorial zone $\mathcal{R}_s$ located at position $d_s$. The complete set of opposing teams is denoted as $\mathcal{K} = \{1, \ldots, K\}$. Each agent possesses a circular interaction zone of radius $R$ within which it can eliminate an opponent agent through direct engagement, resulting in mutual destruction of both participants. Figure~\ref{fig:scenarios} illustrates this multi-objective adversarial configuration.

The information structure enforces partial observability constraints, where each agent $i$ has access to limited observations $z_i \in \mathcal{Z}$ comprising: (i) its current position $p_i$ and velocity $v_i$; (ii) the locations of all safe areas $ds_k$; (iii) a restricted set of teammate positions $\mathcal{O}_s = \{p_j \mid j \in \mathcal{N}_i\}$ within sensing range; and (iv) a limited collection of opponent positions $\mathcal{O}_k = \{p_{o_k} \mid k \in \mathcal{N}_k, \forall k \in \mathcal{K}\}$ across all opposing teams.

The primary objective is to derive a collaborative control policy that enables team $\mathcal{V}_s$ to successfully capture all adversarial safe areas while maintaining control of their own territory under partial observability constraints. 

Each agent receives a multi-dimensional reward signal $\mathrm{rw}_i(s, u): \mathcal{S} \times \mathcal{U} \rightarrow \mathbb{R}^{K+1}$ that evaluates performance across $K$ attack objectives (corresponding to opponent safe areas) and one defense objective (protecting the home territory). The reward structure must carefully balance individual agent contributions with collective team achievements, incorporating coordination incentives for collaborative attack and defense strategies while enforcing operational constraints such as collision avoidance. \black{Full architectural parameters and reward function aggregations are detailed in the Appendix.}

\begin{figure*}[t]
	\centering
	\includegraphics[width=\linewidth]{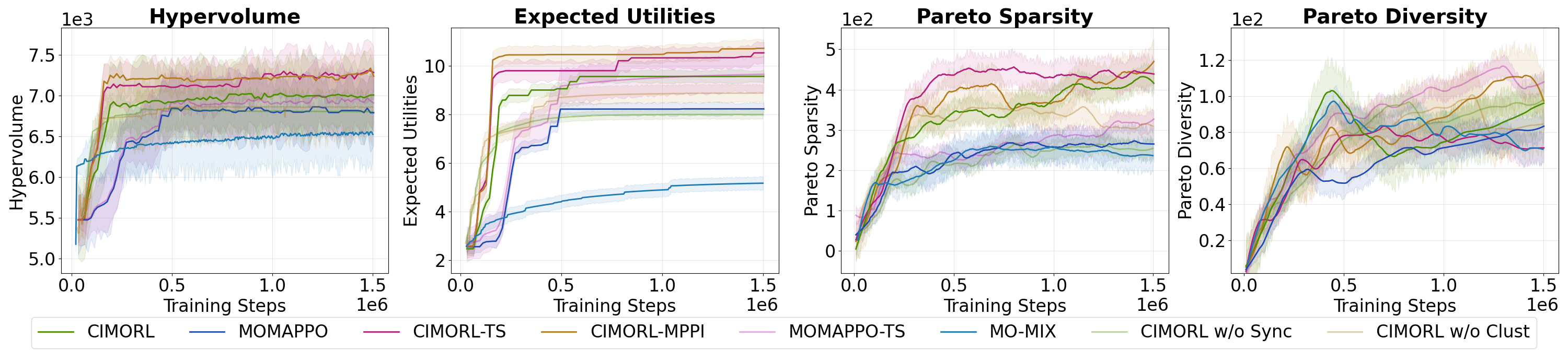}
	\caption{\textbf{Multi-team attacker-defender performance comparison}}
	\label{fig:comparison_ad}
\end{figure*}
\begin{figure*}[t]
	\centering
	\includegraphics[width=\linewidth]{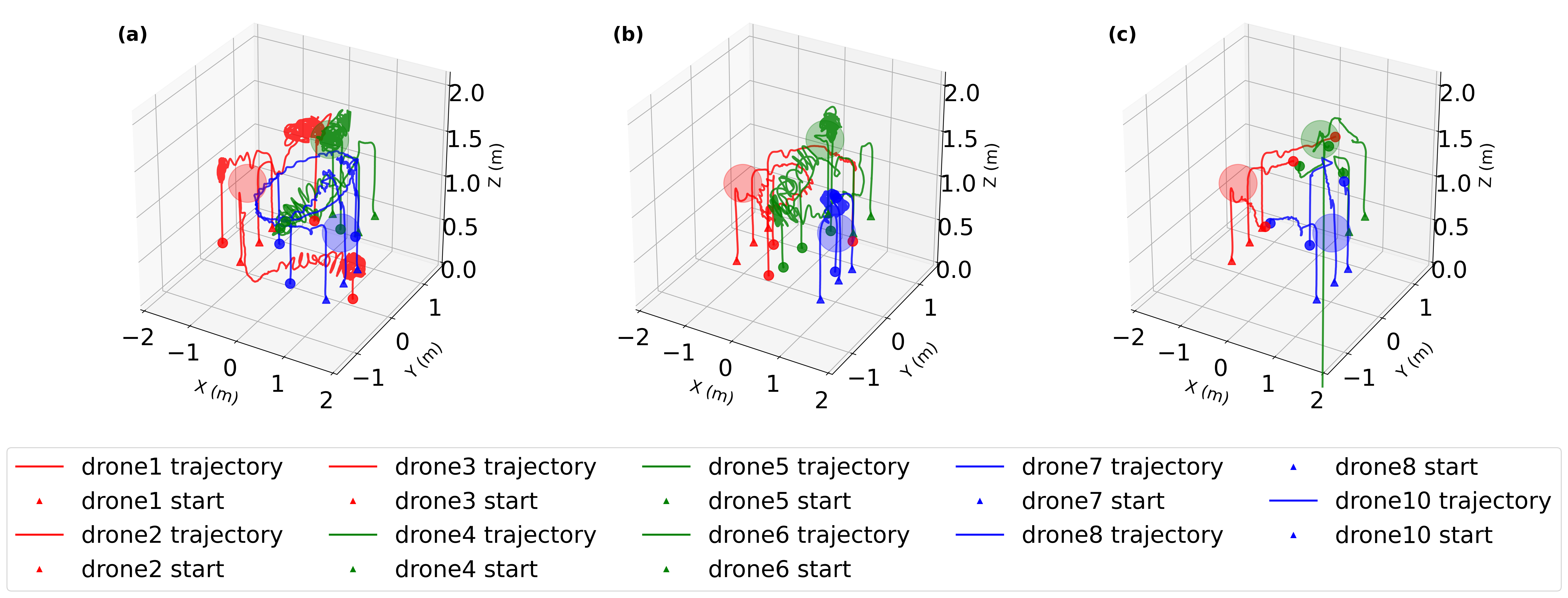}
	\caption{\textbf{3D trajectories in multi-team attacker-defender experiments.}}
	\label{fig:trajectory_comp_ad}
\end{figure*}
\begin{table}[t]
    \centering
    \caption{Multi-Objective Scores with Alternating Policies. The assigned policy rotates across the Blue, Red, and Green teams for each scenario.}
    \label{tab:score_ad}
    \begin{tabular}{l | ccc}
        \toprule
        \textbf{Scenario} & \textbf{Red Team} & \textbf{Green Team} & \textbf{Blue Team} \\
        \midrule
        \textbf{a} & \textit{CIMORL} & \textit{MOMAPPO} & \textit{CIMORL-MPPI} \\
         & $(-694, 992, 379)$ & $(-2499, 0, 0)$ & $(-379, 694, 1507)$ \\
        \addlinespace
        \textbf{b} & \textit{CIMORL-MPPI} & \textit{CIMORL} & \textit{MOMAPPO} \\
         & $(0, 727, 0)$ & $(-727, 0, 0)$ & $(0, 0, 0)$ \\
        \addlinespace
        \textbf{c} & \textit{MOMAPPO} & \textit{CIMORL-MPPI} & \textit{CIMORL} \\
         & $(0, 0, 0)$ & $(-418, 0, 0)$ & $(0, 0, 418)$ \\ 
        \bottomrule
    \end{tabular}
\end{table}
\textbf{Training Results.} For the multi-team attacker-defender scenario, we conducted experiments with $20$ agents distributed equally across $4$ teams, creating a complex adversarial environment with multiple competing objectives. The sampling-based algorithms maintained identical simulation budgets and planning horizons as in the resource assignment experiments, with the same decreasing learning rate schedule to ensure training stability. Figure~\ref{fig:comparison_ad} presents the performance comparison for this challenging scenario, \black{revealing how different algorithms navigate the trade-offs between hypervolume, utility, sparsity, and diversity in competitive dynamics.}

\black{CIMORL-MPPI} achieved the strongest overall performance with a hypervolume of $7.25 \times 10^3 \pm 0.25 \times 10^3$ and expected utility of $11.0 \pm 0.45$. The relatively low variance in both metrics indicates consistent performance across evaluation runs, suggesting that the MPPI-based approach provides stable policy learning in adversarial settings. This stability can be attributed to MPPI's uniform exploration strategy, which systematically samples the action space without making overly optimistic assumptions about opponent behavior. \black{CIMORL-TS} achieved comparable hypervolume performance ($7.25 \times 10^3 \pm 0.33 \times 10^3$) but exhibited slightly lower expected utility ($10.7 \pm 0.5$) and higher variance during training. \black{This performance profile suggests that Monte Carlo Tree Search's optimistic selection strategy drives excellent exploration, establishing the highest upward trajectories in both Pareto sparsity (reaching near $450$) and diversity (peaking above $110$), but leaves the policy more sensitive to unpredictable adversarial shifts, as evidenced by the higher variance} ($\sigma = 0.33$ vs. $0.25$ for hypervolume, $\sigma = 0.5$ vs. $0.45$ for utility).The base CIMORL method demonstrated strong performance with a hypervolume of $7.0 \times 10^3 \pm 0.5 \times 10^3$ and expected utility of $9.5 \pm 0.3$. Notably, the performance gap between CIMORL and the sampling-based variants is smaller in this adversarial scenario compared to the cooperative resource assignment task. \black{Its sparsity and diversity metrics closely mirror those of CIMORL-MPPI, maintaining a strong middle-to-high position throughout training.} 

Among the baseline methods, MOMAPPO-TS achieved superior performance compared to standard MOMAPPO, with hypervolume values of $6.9 \times 10^3 \pm 0.4 \times 10^3$ versus $6.7 \times 10^3 \pm 0.3 \times 10^3$, representing a $3\%$ improvement. MOMAPPO-TS exhibited significantly higher variance in expected utility ($9.6 \pm 1.0$ vs. $8.4 \pm 0.25$), indicating that while tree search integration improves average performance, it also introduces inconsistency in policy execution. \black{MO-MIX is subjected more to the non-stationarity in adversarial games, as it is an off-policy method. Therefore, it reaches the worst performances in all the metrics. Consequently, these baselines exhibit notably lower sparsity, plateauing early in the $250$ to $300$ range, and cluster in the bottom tier for diversity. Additionally, all metrics display wide, heavily overlapping confidence intervals across these algorithms, reflecting the substantial variance inherent in competitive multi-agent environments.} Both CIMORL w/o Sync and CIMORL w/o Clust converge to a similar hypervolume of $6.65 \pm 0.2 \times 10^3$. However, consistent with the previous scenario, the w/o Sync variant yields lower expected utilities compared to the other baselines \black{and suffers a sharp, early plateau in sparsity}. Furthermore, the expected utilities produced by CIMORL w/o Clust clearly illustrate the performance gap between full synchronization and isolated clustering dynamics, amounting to an average difference of approximately $0.8$. \black{Without clustering, the model also remains confined to the lower tier of Pareto diversity, further underscoring the necessity of the proposed coordination mechanisms.}

\black{We also evaluated the proposed approach on larger teams, with results presented in Figure~\ref{fig:ad-team-size}. As with the resource-allocation task, we normalized the returns to a $[0, 1]$ scale. We observed that the baseline CIMORL policy struggles to maintain performance across objectives as the team size increases, regardless of whether the agents are assigned attacker or defender roles. In contrast, both MPPI and tree search successfully sustain their performance. This confirms that the proposed multi-objective online search effectively extends coordination capabilities to larger teams than those encountered during training.}

\textbf{Experimental Comparison.} We conducted experiments with three competing teams, each composed of three drones, rotating the policies among them in three different experiments. Again due to time budget, MPPI uses 1000 simulations with an horizon length of 10 steps for computational efficiency.

Figure~\ref{fig:trajectory_comp_ad} reports the 3D trajectory across the three distinct experimental configurations, each showcasing different team-policy assignments. In experiment (a), the blue team employs CIMORL-MPPI, the green team uses MOMAPPO, and the red team implements CIMORL. In experiment (b), policy assignments rotate with the red team utilizing CIMORL-MPPI, blue team employing MOMAPPO, and green team implementing CIMORL. Experiment (c) completes the rotation with the green team using CIMORL-MPPI, red team employing MOMAPPO, and blue team implementing CIMORL.

Table~\ref{tab:score_ad} quantifies team performance through a comprehensive scoring system that awards at every time step +1 points for each agent successfully infiltrating opponent areas and -1 points for allowing enemy agents into their own territory for the corresponding objectives.

In experiment (a), the blue team (CIMORL-MPPI) demonstrates superior strategic coordination, achieving scores of (-379, 649, 1507) through a circular patrol pattern that simultaneously enables area protection and coordinated attacks. This was an emergent behavior from the policy not pre-programmed or searched during training. This behavior exemplifies the adaptive weight dynamics' capability to balance competing objectives dynamically. The red team (CIMORL) exhibits a more conservative strategy with scores of (-649, 992, 379), distributing drones uniformly among safe areas, which demonstrates effective weight distribution across objectives but lacks a robust coordination pattern. The green team (MOMAPPO) experiences catastrophic failure due to inter-agent collisions, resulting in severely negative performance with a score of (-2499, 0, 0). This failure mode is characteristic of fixed-weight architectures: lacking the ability to dynamically down-weight aggressive territorial infiltration in favor of immediate spatial de-confliction, the agents exhibit rigid, brittle behavior that easily collapses into physical gridlock when facing heterogeneous opponents.

Experiment (b) reveals distinct strategic adaptations, where the red team (CIMORL-MPPI) with scores of (0, 727, 0) develops an advanced enclosure strategy, effectively surrounding and neutralizing green team drones protecting their territory. The blue team (MOMAPPO) suffers immediate collision-induced failure, achieving null scores (0, 0, 0), \black{while the green team (CIMORL), with scores of (-727, 0, 0), demonstrates improved collision avoidance compared to MOMAPPO but fails to mount effective defensive strategies, allowing complete territorial infiltration. Crucially, CIMORL's failure mode is distinct from MOMAPPO's immediate physical gridlock. Instead of colliding, the green team's trajectories exhibit severe spatial oscillation (as seen in Figure ~\ref{fig:trajectory_comp_ad}.b). This oscillatory hesitation occurs because the dynamic weight consensus mechanism is attempting to adapt to the adversarial pressure, but without expert demonstrations, the model has found an oscillatory strategy. The agents constantly shift their priority weights between attacking and defending, delaying decisive cooperative actions. This behavior is consistent with our theoretical analysis in Sec. VI, illustrating a system that is provably contracting, but where the contraction rate is bottlenecked by the baseline training parameters. The sampling-based search in CIMORL-MPPI overcomes this by utilizing privileged expert examples to rapidly smooth the objective landscape, inducing a much higher contraction rate that translates to decisive, rather than hesitant, spatial maneuvers.}

In experiment (c), coordinated aggression from red and blue teams against the green team creates a unique defensive scenario. The green team (\black{CIMORL-MPPI}) adopts a sacrifice-based defensive strategy, concentrating all resources on territorial protection. The blue team achieves modest success with scores of (0, 0, 418) through effective resource allocation, while the red team's escape attempts ultimately fail, resulting in null scores (0, 0, 0). The green team's defensive sacrifice yields scores of (-418, 0, 0), demonstrating \black{CIMORL-MPPI}'s ability to adapt strategy based on overwhelming adversarial pressure.

A critical observation from these real-world experiments is the substantial behavioral differences compared to training scenarios. During training, policies are developed through self-play, where agents learn to coordinate against opponents employing identical strategies and behavioral patterns. However, in the heterogeneous real-world experimental setup, each team employs fundamentally different policies with distinct coordination mechanisms and strategic approaches. Transitioning from homogeneous self-play in training to heterogeneous open-play in reality creates a significant deployment gap. \black{This gap exposes the limitations of traditional reinforcement learning and proves that online adaptation is crucial. Specifically, it reveals two distinct suboptimal failure modes in the baselines: (i) the rigid brittleness of fixed-weight approaches (MOMAPPO), which leads to unresolvable spatial conflicts and collisions, particularly in tight resource-assignment scenarios, and (ii) the oscillatory hesitation of unaugmented dynamic approaches (base CIMORL), which successfully avoid collisions but fail to commit to a decisive strategy due to slow weight consensus.}

The sampling-based search integration in \black{CIMORL-MPPI} demonstrates remarkable online adaptation abilities, enabling real-time strategy adjustment when confronted with unfamiliar opponent behaviors that were not encountered during training. This adaptive capacity explains \black{CIMORL-MPPI}'s consistent superior performance across all experimental configurations, as it can dynamically recalibrate its coordination patterns, weight distributions, and tactical decisions based on observed opponent strategies. In contrast, fixed-weight approaches like MOMAPPO and weight-distribution methods like CIMORL struggle to adapt beyond their training distributions, leading to suboptimal performance or catastrophic failures when facing novel adversarial strategies.

\section{Conclusion}

This work introduced a coordination-informed multi-objective reinforcement learning (CIMORL) framework addressing fundamental challenges in balancing competing objectives within multi-robot systems. The approach integrates three key innovations: distributed weight prediction through opinion dynamics, privileged expert training via sampling-based search methods, and theoretical guarantees for Pareto-optimal coordination through coalition formation.
The experimental validation demonstrated significant improvements over state-of-the-art baselines, with our sampling-based variants demonstrating strong performance in complex real-world drone experiments where baseline methods failed, highlighting the critical role of online adaptation.

Future research directions will include improving the convergence speed of the proposed weight model, which currently converges slowly to the equilibria, by imposing a higher contraction rate. Investigating methods to enforce this higher contraction rate while maintaining the numerical stability of the integration presents a challenging but important direction. Second, we aim to evaluate the implications of the hierarchical scheme in the proposed weight-policy model and explore more efficient training methods to minimize policy variance. \black{Additionally, we plan to conduct a formal theoretical and empirical analysis of the weight adaptation mechanism's robustness against explicit communication delays and network noise. Finally, while our current formulation successfully scales to four simultaneous objectives---including inter-agent collision avoidance---extending this coordination framework to explicitly handle advanced path-planning in cluttered layouts with static and dynamic environmental obstacles remains an exciting avenue for our future research.}

\section{Appendix}
\subsection{Proof to Theorem 1}
\label{sec:app-a}
In this section, we report the proof for Theorem~\ref{th:theorem-2}
\begin{proof} 
The proof strategy is twofold: first, we show that a clustered state, where agents in a coalition share the same trajectory, is a valid equilibrium solution for the system's dynamics; second, we prove that the system is contracting, which guarantees that any initial state will globally converge to such a stable, clustered trajectory.

	Consider $M < N_s$ clusters of team $\mathcal{V}_s$, where
	each cluster $c_r$ contains agents satisfying
	$\text{dist}(\Xi_i, \Xi_j) < \epsilon$ for $i, j \in c_r$
	and sufficiently small $\epsilon > 0$.
	Let the joint state be $x = [x_1, x_2]$, and denote two
	state trajectories as $\xi = [\xi_1,\xi_2]$ and $\bar{\xi} = [\bar{\xi}_1, \bar{\xi}_2]$.

	We define $\bar{\xi}_2$ such that it
	follows the clustered structure: $$\bar{\xi}_2 = [1_{c_1}
		\otimes x_{c_1}, \dots, 1_{c_M} \otimes x_{c_M}]$$ and $\bar{\xi}_1$ to make $\bar{\xi}_2$ an equilibrium trajectory for $x_2$.

	Using the vectorization operator $X_| = \text{vec}(X)$, we
	have $x_| = [x_{1|}, x_{2|}] \in \mathbb{R}^{N_s \cdot 2K}$.
	Define the virtual state $x = \xi + \mu(\bar{\xi} - \xi)$
	with $x(\mu = 0) = \xi$ and $x(\mu = 1) = \bar{\xi}$.
	We now show that $\bar{\xi}$ is a particular solution of
	the system described by Eq.~\eqref{eq:model-dynamics}.
	For agent $i$ and its driving signal $\Xi_i$,
	its equilibrium trajectory is given by:
	\begin{equation}
		\begin{aligned}
			\dot{\bar{\xi}}_{i1} & = (\text{So}_{\Xi_i}\bar{\xi}_2 W_B
			- \bar{\xi}_{i2} W_B)                                      \\
			\dot{\bar{\xi}}_{i2} & = -\left(\tau + \Xi_i \right)
			\circ \bar{\xi}_{i2}+\text{So}_{\bar{\xi}_i}\bar{\xi}_2W_A
			- \bar{\xi}_{i2}W_A                                        \\                  & -\text{So}_{\Xi_i}\bar{\xi}_{i1}W_B
			                   + \bar{\xi}_{i1} W_B + \Xi_iB
		\end{aligned}
	\end{equation}
	where $\text{So}_{\Xi_i}, \text{So}_{\bar{\xi}_i}$ are
	the $i$-th rows of the matrices. Next, we select agent
	$j$ from the same cluster as $i$ and consider the difference
	in their dynamics, $\dot{\bar{\xi}}_i - \dot{\bar{\xi}}_j$,
	which is:
	\begin{equation}
		\label{eq:diff_agents}
		\begin{aligned}
			\dot{\bar{\xi}}_{1i} - \dot{\bar{\xi}}_{1j} = &
			\text{So}_{\Xi_i}\bar{\xi}_2 W_B
			- \text{So}_{\Xi_j}\bar{\xi}_2 W_B              \\
			\dot{\bar{\xi}}_{2i} - \dot{\bar{\xi}}_{2j} = &
			-(\Xi_i\circ \bar{\xi}_{2i}  - \Xi_j \circ \bar{\xi}_{2j})
			+ (\text{So}_{\bar{\xi}_i}\bar{\xi}_2W_A -
			\text{So}_{\bar{\xi}_j}\bar{\xi}_2W_A)          \\                  &
			                   -(\text{So}_{\Xi_i}\bar{\xi}_1W_B
			                   - \text{So}_{\Xi_j}\bar{\xi}_1W_B) + (\Xi_i - \Xi_j)B
		\end{aligned}.
	\end{equation}
	The difference in eq\eqref{eq:diff_agents} is equal to the
	null vector as long as $\text{So}_{\Xi}, \text{So}_{\xi}$
	have constant row sum in the block matrices among the
	clusters, as demonstrated in~\cite{xia2011clustering}.
	Given our definition for these two matrices, this property
	is structurally satisfied.
	Moreover, the differences in eq\eqref{eq:diff_agents} are
	zero when $\bar{\xi}_1$ is such that $(\text{So}_{\Xi i}-
		\text{So}_{\Xi j})\bar{\xi}_1W_B = (\Xi_i - \Xi_j)B
		-(\Xi_i\circ \bar{\xi}_{2i} - \Xi_j \circ \bar{\xi}_{2j})$.
	We now prove that the system is contracting for the two
	trajectories $\xi$ and $\bar{\xi}$.
	The dynamic of the virtual displacement $\delta x_|$  is defined by
	\begin{equation}
		\label{eq:differential-dynamics}
		\delta\dot{x_{|}} =  F \delta x_{|} + G \delta z_{|}
	\end{equation}
	with
	\begin{equation}
		\begin{aligned}
			F & = \begin{bmatrix}
				      0                                       & W_{B}^\top \otimes ( I_N - \text{So}_{\Xi}) \\
				      W_{B}^\top \otimes (I_N - \text{So}_{\Xi}) &
				      -\diag(\tau+\Xi)+W_{A}^\top\otimes (I_N - \text{So}_{x})
			      \end{bmatrix} \\ G & =  \begin{bmatrix}
				W_{B}^\top \otimes (\frac{\partial \text{So}_{\Xi}}{\partial \Xi} \frac{\partial g }{\partial z})  x_{2|} \\
				\diag( \frac{\partial g }{\partial z})x_{2|} + \diag(\frac{\partial g }{\partial z}) B_{|} -W_B^\top \otimes (\frac{\partial \text{So}_{\Xi}}{\partial \Xi}\frac{\partial g }{\partial z}) x_{1|}
			\end{bmatrix}
		\end{aligned}
	\end{equation}

	To establish convergence to the clustered equilibrium $\bar{\xi}$, we analyze the contraction properties of the unforced differential dynamics with respect to the state $x$, given by $\delta\dot{x}_{|} = F \delta x_{|}$. We define the differential Lyapunov function $P$ representing the Riemannian energy of the virtual displacement:
	\begin{equation}
		P = \int_0^1 \delta x_{|}^\top \delta x_{|} d\mu
	\end{equation}
	Taking the time derivative of $P$ along the trajectories of the system yields:
	\begin{equation}
		\dot{P} = \int_0^1 \delta x_{|}^\top (F + F^\top) \delta x_{|} d\mu
	\end{equation}
	The stability of the system is governed by the symmetric part of the Jacobian matrix, $F_s = \frac{1}{2}(F + F^\top)$. Analyzing the block structure of $F$, the top-left block corresponding to the auxiliary state $x_1$ is strictly zero. However, the bottom-right block governing the primary weight state $x_2$ is defined by $F_{22} = -\text{diag}(\tau+\Xi) + W_{A}^\top\otimes (I_N - \text{So}_{x})$. By design, the bias $\tau > 0$, the activation $\Xi \ge 0$, and the attention mechanism $(I_N - \text{So}_{x})$ acts as a network Laplacian. Consequently, $F_{22}$ is strictly Hurwitz, and its symmetric part is negative definite.

	While the presence of the zero block in $F$ implies that the full Jacobian is only negative semi-definite, the system exhibits weak (or partial) contraction. If we consider the projected state transformation $\hat{y} = [0,I] x_{|}^\top = x_{2|}$, its differential dynamic is strongly contracting with a rate bounded by the maximum eigenvalue of the symmetric part of $F_{22}$, such that $\lambda_{\max}(F_{22} + F_{22}^\top) < 0$. As shown in standard partial contraction theory~\cite{dani2014observer}, because the $x_2$ subsystem is strictly contracting, all trajectories of $x_2$ exponentially converge to a single unique trajectory, regardless of their initial conditions.

	This rigorous contraction of the $x_2$ state forms the direct logical connection to our clustering definition. We have already established in Eq.~\eqref{eq:diff_agents} that the clustered trajectory $\bar{\xi}$, where agents in a coalition share the exact same state, is a valid equilibrium solution of the system. Because the system is contracting with respect to $x_2$, the actual network state $x_2(t)$ must asymptotically converge to this clustered equilibrium $\bar{\xi}_2(t)$. Thus, the agents inherently form the clusters specified in Definition~\ref{def:cluster-trajectories}. 
\end{proof}
\subsection{Proof to Lemma 3}
\label{sec:proof-lemma3}
In this section, we report the proof to Lemma~\ref{lemma:bounded-x2}
\begin{proof}
	We establish the boundedness of the system using the quadratic Lyapunov function
	\begin{equation*}
		P_1 = \frac{1}{2}x_{1|}^\top x_{1|} + \frac{1}{2}x_{2|}^\top x_{2|}
	\end{equation*}
	where the notation $(\cdot)_{|} = \text{vec}(\cdot)$ denotes the standard column-wise vectorization operator. This specific choice of $P_1$ is standard in nonlinear systems analysis for establishing $L_2$ stability and proving the existence of bounded invariant sets by evaluating the generalized energy of the states.
	
	Focusing on the $x_2$ dynamics, the time derivative of $P_1$ along the system trajectories yields:
	\begin{equation*}
		\begin{aligned}
			\dot{P}_1 & \le -x_{2|}^\top \text{diag}(\tau_{2|}) x_{2|} - x_{2|}^\top \text{diag}(\Xi_{|}) x_{2|} \\
			          & \quad - x_{2|}^\top (W_A^\top \otimes (I_N - \text{So}_{\Xi})) x_{2|} + x_{2|}^\top \text{diag}(\Xi_{|}) B_{|}
		\end{aligned}
	\end{equation*}
	Assuming the structural matrices associated with $x_{2|}$ and the network coupling $W_A^\top \otimes (I_N - \text{So}_{\Xi})$ are positive semi-definite, their corresponding quadratic forms are non-positive. To prove that $[-B,B]^{N \times K}$ is a positively invariant set, we analyze the remaining terms on the boundary by expanding them element-wise for each index $i \in \{1, \dots, NK\}$:
	\begin{equation*}
		-x_{2|}^\top \text{diag}(\Xi_{|}) x_{2|} + x_{2|}^\top \text{diag}(\Xi_{|}) B_{|} = \sum_{i=1}^{NK} \Xi_{|,i} \left( -x_{2|,i}^2 + B_{|,i} x_{2|,i} \right)
	\end{equation*}
	Consider the case where the state $x_2$ attempts to leave the region $[-B, B]^{N \times K}$. This implies that for at least one element index $i$, the state exceeds the boundary such that $|x_{2|,i}| > B_{|,i}$. Since $B_{|,i} > 0$ by definition, the squared term strictly dominates the linear term: $x_{2|,i}^2 > B_{|,i} |x_{2|,i}| \ge B_{|,i} x_{2|,i}$.
	
	Given that the activation functions yield $\Xi_{|,i} \ge 0$, we can explicitly bound the contribution of the $i$-th element outside the boundary:
	\begin{equation*}
		\Xi_{|,i} \left( -x_{2|,i}^2 + B_{|,i} x_{2|,i} \right) < \Xi_{|,i} \left( -B_{|,i} |x_{2|,i}| + B_{|,i} |x_{2|,i}| \right) = 0
	\end{equation*}
	Because the negative quadratic term strictly dominates the linear input term for any $|x_{2|,i}| > B_{|,i}$, it ensures that $\dot{P}_1 < 0$ strictly outside the defined set. This continuous dissipation of energy on the exterior of $[-B, B]^{N \times K}$ forces the trajectories to remain bounded within it.
	
	Since $x_{2|}$ remains bounded and contractive, the coupling integral is finite:$L_{x_1} = \|W_B\|_{\infty}
		\int_0^t \|(I-\text{So}_{\Xi}(s))x_{2|}(s)\|_{\infty} ds <
		\|W_B\|_{\infty}
		\frac{\|B\|_{\infty}}{\lambda_{\max}(F_{22})} < \infty$
	Therefore, $x_{1|}(t)$ remains bounded in $\mathcal{X}_1$ when initialized within this set.
\end{proof}

\subsection{Proof to Theorem 2}
\label{sec:app-b}
In this section, we report the proof to Theorem~\ref{theorem:closed-loop}
\begin{proof}
To formally prove the closed-loop contraction of the multi-agent system, we analyze the overall architecture as an interconnected system comprised of two subsystems: the observation dynamics~\eqref{eq:ito-stochastic} and the weight prediction dynamics~\eqref{eq:model-dynamics}. The stochastic small-gain theorem~\cite{mironchenko2015construction,russo2012contraction} dictates that an interconnection of two contracting subsystems remains globally contracting if the product of their individual contraction rates is strictly greater than the product of their cross-coupling Lipschitz constants.
Let the individual contraction rate of the weight dynamics be bounded by $\lambda_{\max}(F_{22})$, and let the contraction rate of the observation dynamics be $a = \min_i \alpha_i$. We denote the Lipschitz constant of the observation dynamics with respect to the weights $w$ as $L_w$. The remaining requirement is to explicitly derive $L_G$, which bounds the coupling strength of the weight dynamics with respect to the observation $z$.

Consider the coupling function $G(x)$ derived from the differential dynamics of the system. We must evaluate its Jacobian with respect to the observation vector $z$. Applying the chain rule through the nonlinear mapping $\Xi = g(z)$, the partial derivatives of the effective similarity matrix $\text{So}_{\Xi}$ and the primary state dynamics yield the following block matrix structure for the Jacobian $J = \frac{\partial G}{\partial x}$:
\begin{equation*}
	J = \begin{bmatrix}
		0 & W_{B}^\top \otimes \left(\frac{\partial \text{So}_{\Xi}}{\partial \Xi} \frac{\partial g }{\partial z}\right) & 0 \\
		-W_B^\top \otimes \left(\frac{\partial \text{So}_{\Xi}}{\partial \Xi}\frac{\partial g }{\partial z}\right) &
		\diag\left(\frac{\partial g }{\partial z}\right) & \diag\left(\frac{\partial g }{\partial z}\right)
	\end{bmatrix} \begin{bmatrix}
		x_{1|} \\ x_{2|} \\ B_{|}
	\end{bmatrix}
\end{equation*}
To establish the upper bound $L_G$, we evaluate the infinity norm of this block matrix, $\norm{J}_{\infty}$. We leverage the known analytical bounds of the specific nonlinearities employed in the architecture. Specifically, the gradients of both the standard sigmoid function $\sigma(\cdot)$ and the softmax operation $\text{So}(\cdot)$ are strictly bounded such that their maximum singular values never exceed $\frac{1}{4}$. Additionally, we denote the maximum Lipschitz constant of the observation mapping network $g(z)$ as $L_Z$, meaning $\norm{\frac{\partial g}{\partial z}}_{\infty} \leq L_Z$.

Applying the triangle inequality and these known bounds to the block components, we can bound the supremum of the coupling function's norm across all states $x$ bounded by $B$:
\begin{equation*}
	\begin{aligned}
		\sup_x \norm{G(x)}_{\infty} &\leq \max\left(B,  \|W_B\|_{\infty} \frac{\|B\|_{\infty}}{\lambda_{max}(F_{22})}\right) \\
		&\quad \times L_Z \left( \norm{W}_{\infty} + 1 + \frac{1}{16}\norm{W_B}_{\infty} \right)
	\end{aligned}
\end{equation*}
The factor of $\frac{1}{16}$ emerges directly from the sequential application of the chain rule through both the sigmoid and softmax functions ($\frac{1}{4} \times \frac{1}{4}$). Simplifying the expression relative to the coupling weights yields the final bound:
\begin{equation*}
	L_{G} = 2 L_Z \left( \|W\|_{\infty} + 1 + \frac{1}{16}\|W_B\|_{\infty} \right)
\end{equation*}
With $L_G$ strictly bounding the sensitivity of the weight dynamics to the observations, and $L_w$ bounding the sensitivity of the observations to the weights, the closed-loop system is verified to be stable and contracting if the small-gain condition is satisfied:
\begin{equation*}
	\lambda_{\max}(F_{22}) \cdot a > L_{G} \cdot L_{w}
\end{equation*}
\black{This concludes the proof, demonstrating that as long as the internal consensus forces and local observation contractions dominate the cross-coupling perturbations, the entire multi-agent team will reliably converge to stable weight clusters.}
\end{proof}
\subsection{Hardware Setup and Training Parameters}
The agents are modeled as double integrators operating at 20 Hz, stabilized by a proportional-derivative (PD) policy with stiffness $k_p=7$ and damping $k_d=5$. This yields the closed-loop matrix $A = \left[\begin{smallmatrix} 0 & I \\ -k_p I & -k_d I \end{smallmatrix}\right]$. To satisfy Theorem 2 directly in the physical observation space, we introduce a cross-coupling term $\epsilon = 0.05$ to define the non-diagonal contraction metric $M = \left[\begin{smallmatrix} (k_p + \epsilon k_d)I & \epsilon I \\ \epsilon I & I \end{smallmatrix}\right]$. This specific choice makes the symmetric generalized Jacobian strictly negative definite by perfectly canceling the off-diagonal cross-terms: $A^T M + M A = \left[\begin{smallmatrix} -2\epsilon k_p I & 0 \\ 0 & -2(k_d - \epsilon)I \end{smallmatrix}\right]$. Consequently, the single-agent subsystem is proven to be strictly contracting with a rate of $\alpha = \min(\epsilon k_p, k_d - \epsilon)$. At the same sampling rate, the agents communicate with their neighbors and generate actions from the learned policy. All training and evaluations were conducted on a machine running Ubuntu 22.04 with an Intel Core i7-9750H @ 2.60GHz CPU, Nvidia RTX 2080Ti, and 32GB RAM. The observation vectors were normalized to the range $[-1,1]$, and a communication range of $1$ was considered in the normalized space. A full breakdown of the hyperparameters utilized across algorithms is detailed in Table~\ref{tab:training_parameters}. We chose the training hyperparameters and the coefficients $c_1,c_2,c_3$ leading to the best performances in out configuration by performing parameters sweeps. 
\begin{table*}[h]
	\centering
    \caption{Training parameters}    
	\label{tab:training_parameters}
    \begin{tabular}{l | lllll}
		\hline
		Parameters   & CIMORL|w/o Clust|w/o Sync  & MOMAPPO & MO-MIX & CIMORL-TS|CIMORL-MPPI & MOMAPPO-TS \\ \hline
		$\gamma$            & 0.99           & 0.99  & 0.99 & 0.99 & 0.99 \\
		Episode length      & 250            & 250 & 250  & 250   & 250  \\
		$\lambda_{GAE}$     & 0.95           & 0.95   & - & 0.95 & 0.95 \\
		Batch size          & 200            & 200 & 256 & 200 & 200  \\
		Epochs    & 6 & 6 & 1 & 6 & 6 \\
		rollout buffer      & 2000           & 2000 & - & 1000 & 1000 \\
        replay buffer & - & - & \num{150e3} & - & - \\
		Entropy coefficient & 0.02           & 0.02 & 0.05 & 0.01 & 0.01 \\
		Clip $\epsilon$     & 0.3            & 0.3 & - & 0.3 & 0.3  \\
		Max gradient norm   & 0.5            & 0.5 & 0.5 & 0.5 & 0.5  \\
		Adam learning rate  & \num{2e-3}     & \num{2e-3}·& \num{0.9e-3} & \num{2e-3}  & \num{2e-3}       \\
        Sampled $w^*$ & 500 & 500 & 500 & 500 & 500 \\
        Training steps per weight & - & \num{150e3} & \num{150e3} & - & \num{150e3}  \\ 
		Search horizon      & - & - & - &  25 & 25               \\
		Search simulations  & - & - & -             & 2000 & 2000  \\
		$c_1,c_2,c_3$       & -              & - & - & [1.5, 0.5, 0.75] & [1.5, 0.5, 0.75] \\ \hline
	\end{tabular} 
\end{table*}
\subsection{Network Architectures and Reward Structures}
\paragraph{Multi-Resource Assignment}
The reward function is designed to incentivize agents to approach consumer areas when they possess resources that match the consumers' requirements. Specifically, for each consumer demand $k$, we assign the following reward:
\begin{equation*}
	\mathrm{rw}_{ik} = \begin{cases}
		\beta ( \|p_i^{t-1} - ds_k\| - \|p_i^{t} - ds_k\|) & \text{if } \|p^{t}_i - ds_k\| \geq R_k                 \\
		    & \text{and } \mathrm{so}_i \cdot \mathrm{de}_k > 0 \\
		\mathrm{RW}_{ik}                                & \text{otherwise}
	\end{cases}
\end{equation*}
where $\beta > 0$ is a scaling factor and $\mathrm{RW}_{ik} > 0$ represents a positive reward assigned to agent $i$ upon successfully delivering a resource to consumer area $k$. 
The collision avoidance reward is defined as $\mathrm{rw}_{i(K+1)} = $
\begin{equation*}
	\begin{cases}
		-\mathrm{RW}_{cl}                                     & \text{if } \exists j \neq i, \|p_i - p_j\| < R_{cl} \\
		\gamma_{cl} (\min_{j} \|p_i - p_j\| - 3 \cdot R_{cl}) & \text{if } \min_{j} \|p_i - p_j\| < 3 \cdot R_{cl}  \\
		0                                                     & \text{otherwise}
	\end{cases}
\end{equation*}
where $\mathrm{RW}_{cl} > 0$ is a fixed penalty for collisions and $\gamma_{cl} > 0$ is a scaling factor. The connectivity reward is defined as $\mathrm{rw}_{i(K+2)} = -\mathrm{RW}_{co}$ if the $i$-th row of the communication adjacency matrix $S$ contains only zeros, and $0$ otherwise, where $\mathrm{RW}_{co} > 0$ represents the connectivity penalty. We have experimented with $\mathrm{RW}_{co} = 0.4$, $\mathrm{RW}_{cl} = 15$, $R_{cl} = 0.05 $, $\gamma_{cl}=0.1$, $\beta=0.1$, $\mathrm{RW}_{ik} = 15$ and $R_k = 0.2$.

Let $\mathrm{sd}_{ik} = \frac{\mathrm{so}_i \cdot \mathrm{de}_k}{\|\mathrm{so}_i\| \cdot \|\mathrm{de}_k\| + \epsilon}$ indicating whether the agent resource or demand resource are zero for a small $\epsilon > 0$. The agent observation vector $z_i$ is structured as follows:
\begin{equation*}
	z_i = \left[ v_{i}, \{p_i - ds_k, \mathrm{sd}_{ik} \cdot(\mathrm{so}_i + \mathrm{de}_k)\}_{k \in \mathcal{K}}, \{p_i - p_j\}_{p_j \in \mathcal{O}_s} \right]
\end{equation*}
We employ a DeepSets~\cite{zaheer2017deep} architecture to achieve permutation-invariant aggregation across each observation set. For teammate relative positions, we compute:
\begin{equation}
	\Phi_s = \phi_2\left(\sum_{p_j \in \mathcal{O}_s} \phi_1(p_i - p_j) \right)
\end{equation}
where $\phi_1(\cdot)$ and $\phi_2(\cdot)$ represent MLPs mapping inputs to an $F$-dimensional feature space. Similarly, we define aggregation functions $\Phi_k$ for each resource set $k$. The resulting features from different sets are concatenated and processed in an objective-specific manner.

For the $K$ resource-specific objectives, each objective $k$ concatenates the relative distance to the consumer location $p_i - ds_k$ with the aggregated resource features $\Phi_k(\mathrm{sd}_{ik} \cdot (\mathrm{so}_i + \mathrm{de}_k))$. The collision avoidance and connectivity maintenance objectives utilize the teammate aggregation features $\Phi_s$. All objectives incorporate velocity information to enable dynamic motion planning and coordination.

The concatenated features are processed through a final MLP $g(\cdot)$ as defined in equation~\eqref{eq:Xi-def}. This modular observation processing architecture enables flexible objective scaling between training and deployment phases, which may involve different numbers of consumer locations, resource types, and team sizes. Alternatively, a unified MLP $g(\cdot)$ can directly extract fixed-dimensional objective features from the complete observation vector $z_i$. The value function $V_{\phi}$ utilizes an identical architecture to estimate objective-specific episode values. The policy network integrates the sampled weight vector $w$ with the extracted objective features into a unified representation, which is then processed by a final MLP to generate action distributions. All MLPs have $[32,32]$ features and tanh activation function for which is possible to compute the Lipschitz constant easily. All the methods in the following use the same networks with the same dimensions.

\paragraph{Multi-team Attackers-Defenders}
The reward function is formalized through a team-level structure by defining $\mathrm{RW}_s \in \mathbb{R}^{K+1}$ as the aggregated reward vector for team $\mathcal{V}_s$ at time $t$, treating the collective as a unified strategic entity. The problem structure naturally leads to a generalized zero-sum multi-player game formulation when viewed from the team perspective. For each attack objective $k$, we assign a positive reward $\mathrm{RW}_{sk} > 0$ to team $\mathcal{V}_s$ upon successful penetration of opponent region $\mathcal{R}_k$ by any team member. Correspondingly, the defending team $\mathcal{V}_k$ receives $\mathrm{RW}_{kk} = -\mathrm{RW}_{sk}$, establishing the adversarial reward structure. For the defense objective $s$, the roles are reversed: team $\mathcal{V}_s$ seeks to minimize intrusions while opposing teams attempt to maximize their territorial gains. The team's strategic goal is to simultaneously optimize all $K+1$ objectives, contingent upon sufficient agent allocation and coordination capabilities.

The transition from team-level to individual agent rewards requires careful consideration of the partial observability constraints. To keep consistency with our observability constraints, we implement a localized reward structure where agent $i$ receives $\mathrm{rw}_{ik} > 0$ upon successfully infiltrating opponent region $\mathcal{R}_k$, and $\mathrm{rw}_{is} < 0$ when detecting enemy penetration of the home territory $\mathcal{R}_s$ within its observation range. Additionally, agents incur a substantial penalty $\mathrm{rw}_{i} \ll 0$ for teammate collisions across all objective rewards to promote collision-free coordination. We assigned $\mathrm{rw}_{i} = -20$ and $ -\mathrm{rw}_{is} = \mathrm{rw}_{ik} = 1$

The agent observation vector $z_i$ is structured as follows:
\begin{equation*}
	z_i = \left[v_i, \{p_i - ds_k\}_{k \in \mathcal{K} \cup \{s\}}, \{p_i - p_j\}_{p_j \in \mathcal{O}_s}, \{p_i - p_k\}_{p_k \in \mathcal{O}_k}  \right]
\end{equation*}
This representation encodes relative positional information across three distinct categories: distances to all safe areas, relative positions of observable teammates, and locations of detected opponents. These categorical groups undergo separate processing to extract objective-specific features. For attack objectives, we combine relative positions to target safe areas with corresponding opponent locations to form offensive feature representations. The defensive objective aggregates the home safe area position with all observable opponent locations. Teammate positions and agent velocity are incorporated across all objectives to enable coordination. This design ensures that each attack objective focuses exclusively on its designated opponent team, treating other teams as neutral entities that may serve as obstacles or potential allies.

As in Section~\ref{sec:resource-assignment}, we employ DeepSets to process the position differences with permutation-invariant aggregation. The resulting aggregated features are concatenated and processed through the MLP $g(\cdot)$ to create objective-specific features. For this scenario, the value function $V_{\phi}$ shares the same architecture, and the policy network integrates the sampled weight vector $w$ with the extracted objective features before processing them with a final MLP to generate action distributions. All MLPs have $[32,32]$ features and tanh activation function. The same neural networks are used in all the methods compared.

As this environment involves adversarial agents, the max value in equation~\eqref{eq:monte-carlo-value} and equation~\eqref{eq:mppi-value} could fail to capture the worst-case opponent behavior and the policy could loose robustness resulting too optimistic. Normally, this problem is addressed in tree search methods by making the agent playing in turns, with a turn corresponding to one node expansion. However, this strategy works only for discrete games with a full separation of agents turns, like chess or go. In our case, the agents are moving simultaneously and the tree search explores simultaneously the actions of all agents. Moreover, applying a strict minimax formulation often leads to overly pessimistic, degenerate policies in such continuous environments, effectively paralyzing the agents. Finally, MPPI in general is not designed to handle adversarial problems so there is no standard way to address this problem.

Therefore, rather than assuming a strict worst-case opponent, we introduce a risk-sensitive regularized optimization. We slightly modify the max value characterization in both the sampling-based search to penalize deviations of the opponent team values from a nominal value $v_0$, which represents an expected-case, boundedly rational opponent. Given the state of the observed adversarial agents reconstructed by the observation vector, the nominal value $v_0$ is computed as the expected value of the opponent observed teams at the beginning of the search. Moreover, to assign the objective preference to the opponent teams, we consider the average weight vector $\hat{w}$ between agent $i$ solving the tree search and neighboring teammates, computed by agent $i$ from the communicated state of the weight predictor. Let's denote all the opponents agents observed by agent $i$ as $N_{-s}$, we report the modified max value for the tree search:
\begin{equation*}
	\begin{aligned}
		\hat{w}    & = \frac{1}{N_s}\sum_{j=1}^{N_s} w_j \\
		V_{\max}(n) & =                                   \\ & \max_i \sum_{j=1}^{N_s}w_{ij} \frac{|V_{ij}(n) - v^*|}{v^*} + p \sum_{j=1}^{N_{-s}}\Bigg|\hat{w}^\top_i \bigg(V_{ij}(n) - v_0 \bigg) \Bigg|
	\end{aligned}
\end{equation*}
where $p$ acts as a conservatism coefficient that dictates the algorithm's risk aversion. A value of $p = 0.5$ was chosen empirically to balance naive optimism against strict minimax pessimism. We note that a comprehensive game-theoretic analysis of this adversarial adaptation and an ablation over $p \in [0,1]$ are scoped as future work, as the primary focus of this paper remains the cooperative multi-objective framework.

\bibliographystyle{IEEEtran}
\bibliography{IEEEabrv,bibliography}

\begin{IEEEbiography}[{\includegraphics[width=1in,height=1.25in,clip,keepaspectratio]{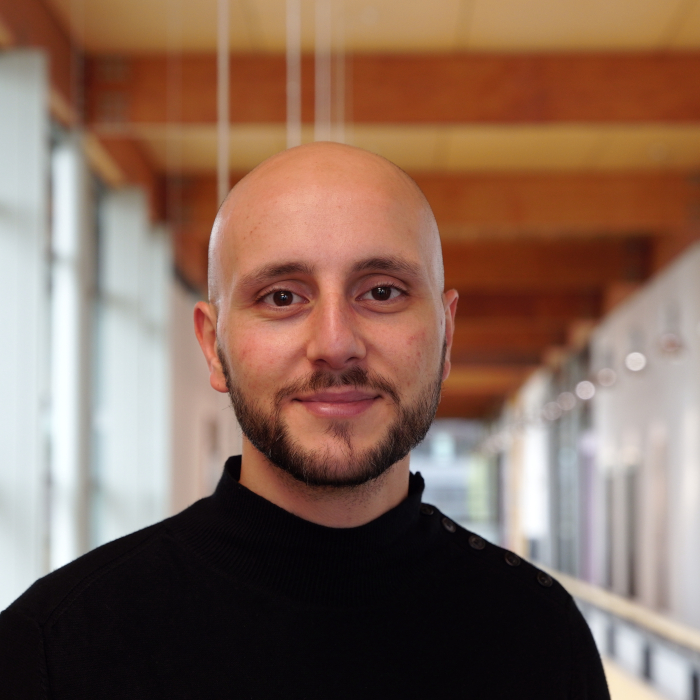}}]{Antonio Marino}
received B.Sc. degree in Automation Engineering from University of Naples "Federico II" in 2018, and his M.Sc. degree in Robotics Engineering from University of Genova in 2020. He obtained the Ph.D. in Robotics Engineering from the University of Rennes, France in 2025. He is currently postdoctoral researcher at the University of Cambridge, UK. \end{IEEEbiography}
\begin{IEEEbiography}[{\includegraphics[width=1in,height=1.25in,clip,keepaspectratio]{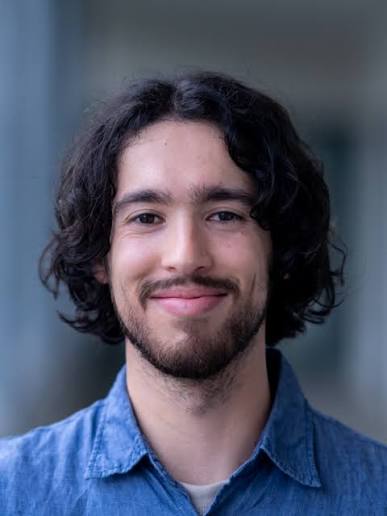}}]{Esteban Restrepo} is currently a research scientist at CNRS-IRISA in Rennes. He received the B.Sc. degree in Mechatronics Engineering from Universidad EIA, Envigado, Colombia and from Arts et Métiers, Paris, France, in 2017, and the M.Sc. degree in Robotic Systems Engineering from Arts et Métiers, Paris, France, in 2018. He obtained the Ph.D. degree in Automatic Control from the University Paris-Saclay, France, in 2021. From January 2022 to January 2023, he was a Postdoctoral Researcher with the Division of Decision and Control Systems, KTH Royal Institute of Technology, Stockholm, Sweden and from January 2023 to November 2024 he was a postdoctoral Researcher with CNRS, and Inria Rennes, France. His research interests include control of multi-agent systems with application to autonomous robots and aerospace systems.\end{IEEEbiography}
\begin{IEEEbiography}[{\includegraphics[width=1in,height=1.25in,clip,keepaspectratio]{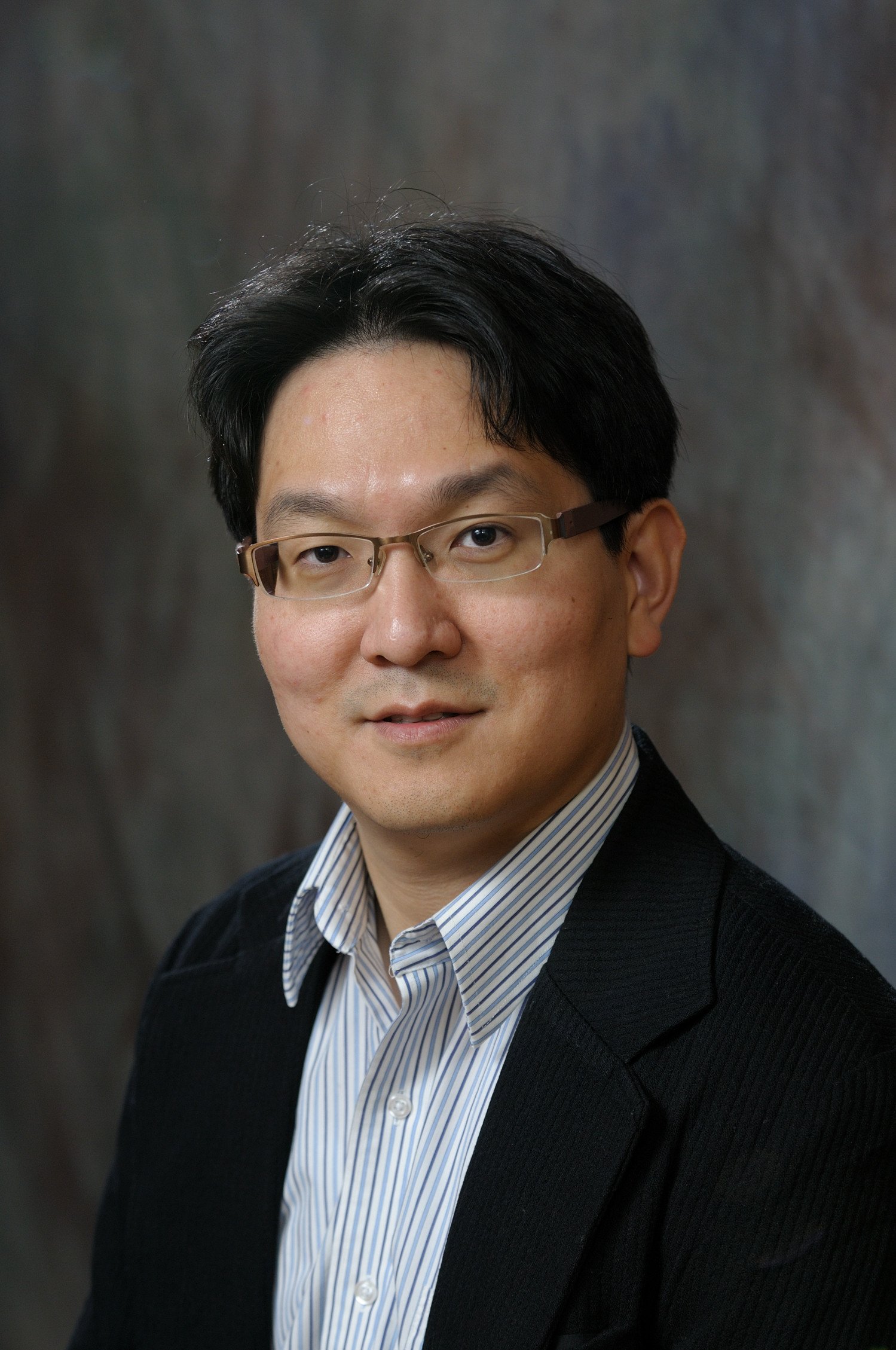}}]{Soon-Jo Chung} (M’06–SM’12) received the S.M. degree in aeronautics and astronautics and the Sc.D. degree in estimation and control from Massachusetts Institute of Technology, Cambridge, MA, USA, in 2002 and 2007, respectively. He is currently Bren Professor of Control and Dynamical Systems, and a Jet Propulsion Laboratory Senior Research Scientist at the California Institute of Technology, Pasadena, CA, USA. He is a Senior Editor of IEEE Robotics and Automation Letters and the IEEE Transactions on Robot Learning. He was an Associate Editor of IEEE Transactions on Robotics, and the Guest Editor of a Special Section on Aerial Swarm Robotics published in IEEE Transactions on Robotics.\end{IEEEbiography}
\begin{IEEEbiography}[{\includegraphics[width=1in,height=1.25in,clip,keepaspectratio]{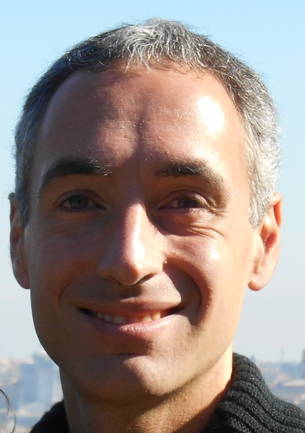}}]{Paolo Robuffo Giordano}  (M'08-SM'16) received his M.Sc. degree in Computer Science Engineering in 2001, and his Ph.D. degree in Systems Engineering in 2008, both from the University of Rome ``La Sapienza''. In 2007 and 2008 he spent one year as a PostDoc at the Institute of Robotics and Mechatronics of the German Aerospace Center (DLR), and from 2008 to 2012 he was Senior Research Scientist at the Max Planck Institute for Biological Cybernetics in T\"ubingen, Germany. He is currently a senior CNRS researcher head of the Rainbow group at Irisa and Inria in Rennes, France.\end{IEEEbiography}
\begin{IEEEbiography}[{\includegraphics[width=1in,height=1.25in,clip,keepaspectratio]{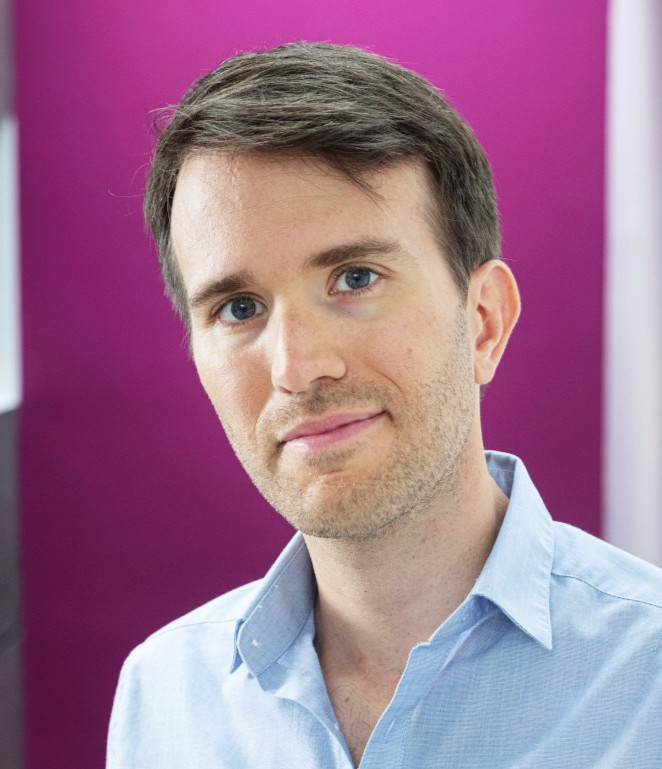}}]{Claudio Pacchierotti} (SM'20) is a
director of research at CNRS-IRISA in Rennes, France, since 2016. He was previously a postdoctoral researcher at the Italian Institute of Technology, Genova, Italy. Pacchierotti earned his PhD at the University of Siena in 2014. He was Visiting Researcher in the Penn Haptics Group at Univ. Pennsylvania in 2014, the Dept. of Innovation in Mechanics and
Management at Univ. Padua in 2013, the Institute for Biomedical Technology and Technical Medicine (MIRA) at University of Twente in 2014, and the Dept. Computer, Control and Management Engineering of the Sapienza Univ. Rome in 2022. Pacchierotti received the 2014 EuroHaptics Best PhD Thesis Award and the 2022 CNRS Bronze Medal. He is Senior Chair of the IEEE Technical Committee on Haptics, Co-Chair of the IEEE Technical Committee on Telerobotics, and Secretary of the Eurohaptics Society.\end{IEEEbiography}

\end{document}